\documentclass[twoside]{article}

\usepackage[accepted]{aistats2025}
\usepackage{layout}
\usepackage{hyperref}

\usepackage{lipsum}
\usepackage{amsfonts}
\usepackage{epstopdf}
\usepackage{algorithmic}
\usepackage{enumitem}
\usepackage[export]{adjustbox}
\usepackage{graphicx}
\usepackage{subfig} 
\usepackage{ragged2e}

\usepackage{bm,amsmath,amssymb,enumitem,booktabs,url,enumitem}
\usepackage[mathscr]{eucal}       
\usepackage{srcltx} 
\usepackage[margin=2cm]{geometry} 

\usepackage{amsthm}
\usepackage{appendix}

\usepackage[all]{xy}  
\usepackage{colortbl} 

\usepackage{mathtools}
\mathtoolsset{showonlyrefs} 
\usepackage{lipsum}
\usepackage{amsfonts}
\usepackage{algorithm}
\usepackage{epstopdf}
\usepackage{algorithmic}

\setlist[enumerate]{leftmargin=.5in}
\setlist[itemize]{leftmargin=.5in}
\setenumerate{labelindent=0em,leftmargin=1.6em,topsep=0cm,partopsep=0cm,parsep=0cm,itemsep=1mm}

\setitemize{labelindent=0em,leftmargin=1.3em,topsep=0cm,partopsep=0cm,parsep=0cm,itemsep=1mm}

\newtheorem{lemma}{Lemma}
\newtheorem{theorem}{Theorem}

\newcounter{aequation} 
\newenvironment{aequation}{%
  \[ 
  \tag{$\mathcal{P}_\k$}
  }{\]}
  
\newcommand{\tb}{\textbf}      
\renewcommand{\b}{\mathbf}     
\newcommand{\R}{\mathbb{R}}    
\newcommand{\Rnn}{\R^{\ge 0}}  
\newcommand{\N}{\mathbb{N}}    
\newcommand{\X}{\mathscr{X}}   
\renewcommand{\d}{\mathrm{d}}  
\newcommand{\E}{\mathbb{E}}    
\renewcommand{\P}{\mathbb{P}}  
\newcommand{\Q}{\mathbb{Q}}    
\newcommand{\W}{\mathscr{W}}   
\newcommand{\I}{\mathbb{I}}    
\newcommand{\Mp}[1]{\mathscr{M}^+_1\left(#1\right)} 
\renewcommand{\k}{k} 
\newcommand{\kk}{k} 
\renewcommand{\H}{\mathscr{H}} 
\newcommand{\Hk}{\H_\k} 
\newcommand{\Bk}{\mathscr{B}_\k} 
\DeclareMathOperator{\MMDo}{MMD}  

    \usepackage{tikz-cd}
    \newcommand\tabnode[1]{\begin{tabular}{c}#1\end{tabular}}

\newcommand{\MMD}[1]{\MMDo_\k\left(#1\right)}     
\newcommand{\MMDs}[1]{\MMDo_\k^2\left(#1\right)}     
\newcommand{\MMDshatU}[1]{\widehat{\MMDo}_{\k,U}^2\left(#1\right)}     
\newcommand{\MMDshatV}[1]{\widehat{\MMDo}_{\k,V}^2\left(#1\right)}     
\newcommand{\MMDshatE}[1]{\widehat{\MMDo}_{\k,e,U}^2\left(#1\right)}     
\newcommand{\MMDshatEU}[1]{\widehat{\MMDo}_{\k,e,U}^2\left(#1\right)}     
\newcommand{\MMDshatEV}[1]{\widehat{\MMDo}_{\k,e,V}^2\left(#1\right)}     
\newcommand{\MMDhatEV}[1]{\widehat{\MMDo}_{\k,e,V}\left(#1\right)}     

\newcommand{\MMDshatUsymb}{\widehat{\MMDo}_{\k,U}^2}
\newcommand{\MMDshatVsymb}{\widehat{\MMDo}_{\k,V}^2}
\newcommand{\MMDshatEsymb}{\widehat{\MMDo}_{\k,e,U}^2}
\newcommand{\MMDshatEVsymb}{\widehat{\MMDo}_{\k,e,V}^2}

\newcommand{\WD}{\text{W}_p}    
\newcommand{\WDhat}{\widehat{\text{W}}_p}    
\newcommand{\iid}{\stackrel{\text{i.i.d.}}{\sim}} 
\renewcommand{\O}{\mathcal{O}} 
\renewcommand{\o}{o} 
\newcommand{\cN}{\mathcal{N}} 

\newcommand{\Nt}{\tilde{N}}
\newcommand{\mP}{\mathcal{P}} 

\newcommand{\cskew}{s} 
\newcommand{\cbs}{r } 
\newcommand{\ckg}{{c}} 

\newcommand\smallO{
  \mathchoice
    {{\scriptstyle\mathcal{O}}}
    {{\scriptstyle\mathcal{O}}}
    {{\scriptscriptstyle\mathcal{O}}}
    {\scalebox{.7}{$\scriptscriptstyle\mathcal{O}$}}
  }

\usepackage{mathtools}
\mathtoolsset{showonlyrefs} 

   \newcommand{\bw}{\b{w}}        
   \newcommand{\br}{\b{r}}        

  \newcommand{\cF}{\mathcal{F}}        

\DeclareMathOperator*{\argmax}{arg\,max}
\DeclareMathOperator*{\argmin}{arg\,min}

\DeclareMathOperator*{\Diag}{Diag}
\newcommand{\Rpp}{\R^{>0}}         

\renewcommand{\t}{\text}

\usepackage[round]{natbib}

\begin{document}

%

%

\twocolumn[
\aistatstitle{Keep it Tighter -- A Story on Analytical Mean Embeddings}
\aistatsauthor{ Linda Chamakh \And Zolt{\'a}n Szab{\'o} }

\aistatsaddress{ Model Risk Governance and Review \\ JP Morgan  Chase \\ London, UK \And Department of Statistics \\ London School of Economics \\ London, UK }]

\begin{abstract}
Kernel techniques are among the most popular and flexible approaches in data science allowing to represent probability measures without loss of  information under mild conditions. The resulting mapping called mean embedding gives rise to a divergence measure referred to as maximum mean discrepancy (MMD) with existing quadratic-time estimators (w.r.t.\ the sample size) and known convergence properties for bounded kernels.
In this paper we focus on the problem of MMD estimation when the mean embedding of one of the underlying distributions is available analytically. Particularly, we consider distributions on the real line (motivated by financial applications) and prove tighter concentration for the proposed estimator under this semi-explicit setting; we also extend the result to the case of unbounded (exponential) kernel with minimax-optimal lower bounds. We demonstrate the efficiency of our approach beyond synthetic example in three real-world examples  relying on one-dimensional random variables: index replication and calibration on loss-given-default ratios and on S\&P 500 data.

\end{abstract}


\section{INTRODUCTION}

Kernel methods \citep{steinwart08support,paulsen16introduction} form one of the most powerful tools in machine learning and statistics with a wide range of successful applications. 
The impressive modelling power and flexibility of kernel techniques in capturing complex nonlinear relations originates from the richness of the 
underlying function class called reproducing kernel Hilbert space \citep[RKHS]{aronszajn50theory} associated  to a kernel.

Kernel functions can be used to capture the similarity of objects belonging to various domains including sequences \citep{kiraly19kernel}, sets \citep{haussler99convolution}, and graphs \citep{borgwardt20graph}. Having a notion of inner product realized by kernels, one can represent probability distributions on any kernel-endowed domain via mean embeddings \citep{berlinet04reproducing,smola07hilbert}, which specifically allows to quantify the divergence between distributions by considering the RKHS distance between their corresponding mean embeddings. The resulting (semi-)metric called  maximum mean discrepancy \citep[MMD]{smola07hilbert,gretton12kernel} forms one of the most popular divergence measures in machine learning; the equivalent \citep{sejdinovic13equivalence} notion in the statistic community is referred to as energy distance \citep{baringhaus04new,szekely04testing,szekely05new} or N-distance \citep{zinger92characterization,klebanov05ndistance}. 

The wide popularity of MMD stems from (i) the computational tractablity of its different estimators, (ii) the existence of closed-form expressions for MMD in case of certain kernel-distribution pairs, (iii) its theoretical guarantees facilitated by the underlying Hilbert structure of RKHSs including concentration properties for bounded kernels \citep{gretton12kernel} and (iv) MMD being a metric for characteristic kernels \citep{fukumizu08kernel,sriperumbudur10hilbert,szabo2018characteristic}. These favorable properties of MMD have given rise to various successful applications, including for instance 
two-sample testing  \citep{gretton12kernel,schrab22efficient,hagrass2024spectral}, independence \citep{gretton08kernel,deb2020measuring,albert22adaptive}, goodness-of-fit testing \citep{balasubramanian21optimality,baum22kernel}, and 
statistical inference \citep{briol19statistical,alquier2024universal}, among many others.

In statistical modelling, the problem of parametric estimation \citep{casella2024statistical}---which aims to find the optimal parametric distribution from a specified family, given a set of observations---is one of the most fundamental tasks. The problem can be tackled by minimizing a divergence (also referred to as a calibration metric)  between the target parametric distribution and the empirical distribution associated to the data, with MMD as a natural choice. In particular, p-value and acceptance region derived by \citet{gretton12kernel} can be used to assess the calibration quality and the adequacy of the chosen distribution family. Moreover, for certain kernels and parametric distributions, the mean embedding can be computed analytically \citep[Table 1]{briol2019probabilistic}. Our primary aim in this paper from theoretical perspective is to understand the benefits of such analytical knowledge (when available): we propose a new semi-explicit (one-sample) MMD estimator and prove its tightened convergence guarantees compared to its classical two-sample counterpart.

From practical angle,  in finance, parametric distributions are widely applied for modelling, simulation and interpretation purposes. Common distributions arising in finance include (i) the beta distribution with a bounded support in $(0,1)$ which is particularly well-suited to model financial ratios such as loan recovery rates \citep{chen2013curve} and (ii) the Gaussian distribution which is often used to model the distribution of the log-returns of stock values starting from the seminal work of \citet{blackscholes1973}. Distributions with non-zero skewness (the normalised third moment of the distribution) are also relevant since financial returns can divert from the Gaussian distribution by exhibiting fat tails and negative skewness \citep{cont01empirical}. Relaxation of the Gaussian distribution in these directions, such as the skew-normal distribution, turned out  \citep{Adcock2015skewfinance} to play a key role in the area.

Two key financial applications with the aforementioned distributions and our motivation from practical point of view are as follows.

\begin{enumerate}[labelindent=0em,leftmargin=1.6em,topsep=0cm,partopsep=0cm,parsep=0cm,itemsep=1mm]   
\item The index replication problem consists of 
    finding the weighted average of individual stocks matching a return distribution \citep{Bamberg2000replication,roncalli2009tracking}. 
As proposed by 
 \cite{chalabi12portfolio,lassance19information}, the problem can be solved by  minimizing the divergence between the distribution of weighted stocks and  that of of index returns. When performing such replication, Gaussian distribution and its relaxations constitute a natural choice for the distribution of returns \citep{blackscholes1973,cont01empirical}.
    \item The modelling of  the loss-given-default (LGD; which represents the percentage of the loan the client or company is not able to repay given he has defaulted) can be assumed to follow a beta distribution, as advised by the financial agency Moody's in their widely-used recovery model methodology \citep{gupton2002losscalctm}. 
\end{enumerate}
In both of these applications and throughout the paper we focus on distributions on the real line. In addition we note that further parametric estimation problems in finance arise with  processes driven by a stochastic differential equation  \citep{bishwal2007parameter}, quantile estimation, and tail dependence modelling \citep{jadhav2009parametric,fortin2002tail}. 

Our \tb{contributions} can be summarized are as follows.
\begin{enumerate}
\item We propose the semi-explicit MMD estimator (relying on analytical mean embedding when available), prove its tighter concentration properties (Theorem~\ref{thm:concentration_explicit}) compared to its two-sample counterpart (established for bounded kernels) and extend the analysis to the unbounded exponential kernel (Theorem~\ref{prop:concentration_exponentiated}) with matching minimax lower bound (Theorem~\ref{th:minimax_unbounded}).

\item Accompanying our tighter concentration analysis, we derive the analytical mean embedding for  new kernel-distribution pairs motivated by financial applications, covering the (Gaussian exponentiated, Gaussian) and (Mat{\'e}rn, beta) pairs.
    \item  We demonstrate the efficiency of our MMD estimator in three applications: a synthetic example,  index replication, and calibration on LGD ratios and on S\&P 500 data. 
\end{enumerate}
The paper is structured as follows. Notations are introduced in Section~\ref{sec:notations}. Section~\ref{sec:MMD-estimators} is dedicated to existing and the proposed semi-explicit MMD estimator. Our theoretical results are presented in 
Section~\ref{sec:results}. Numerical illustrations form the focus of Section~\ref{sec:experiments_finance}. Proofs are deferred to the supplement.

\section{NOTATIONS} \label{sec:notations}
This section is dedicated to definitions and to the introduction of our quantities of interest: mean embedding, maximum mean discrepancy, and our choice of studied kernel and distributions. We introduce the \tb{notations}: $\N$, $\N^*$, $[N]$, $\Rpp$, $\Rnn$, $\b v^{\top}$, $\Diag(\b{v})$, $a \wedge b$, $a\vee b$, $\I_{A}$, $L_{(s)}$, $\O(\cdot)$, $\o(\cdot)$, $\O_{a.s}(\cdot)$, $\o_{a.s}(\cdot)$, $\W^d$, $\zeta_\Q$, $\Phi$, $B$, $\Hk$, $\varphi_k$, $\Bk$, $\Mp{\X}$, $\mu_\k$, $\MMDo_k$.

Natural numbers are denoted by $\N=\{0,1,\ldots\}$; $\N^* = \{1,2,\ldots\}$ stands for the set of positive integers. For $N\in \N^{*}$, $[N]=\{1,\ldots,N\}$. Positive reals are denoted by $\Rpp$; $\Rnn$ stands for non-negative reals. The transpose of a vector $\b  v\in\R^d$ is denoted by $\b v^{\top}$; the diagonal matrix formed of a vector $\b v\in\R^d$ is given by  $\Diag(\b{v})\in\R^{d\times d}$.
The minimum of two numbers $a,b\in \R$ is denoted by $a \wedge b$; their maximum is $ a\vee b$. For a set $A$, $\I_{A}$ is the indicator of $A$: $\I_{A}(x)=1$ if $x\in A$; $\I_{A}(x)=0$ otherwise. Given $(L_s)_{s\in [S]} \subset \R$, the associated order statistics are $L_{(1)}\le \ldots \le L_{(S)}$. 
 The notation $b_n = \O(a_n)$ (resp.\ $b_n = \o(a_n)$) means that $(\frac{b_n}{a_n})_{n\in \N}$ is bounded (resp.\ $\lim_{n\rightarrow \infty}\frac{b_n}{a_n}=0$). For random variables $X_n = \O_{a.s}(a_n)$ (resp.\ $X_n = \o_{a.s}(a_n)$) means that $\left(\frac{X_n}{a_n}\right)_{n\in \N}$ is bounded (resp.\ converges to zero) almost surely. 
The $(d-1)$-dimensional simplex is 
$\W^d = \left\{\bw \in \left(\Rnn\right)^d\,:\, \sum_{j=1}^d w_j = 1\right\}. $
Let $m_\Q$ and $\sigma_\Q$ denote the expectation and the standard deviation of a real-valued random variable with distribution $\Q$; its skewness is defined as the standardized third moment $\zeta_\Q = \mathbb{E}_{x \sim \Q}\left[\left({  {(x-m_\Q) }/{\sigma_\Q }}\right)^{3}\right]$. 
The cumulative density function (cdf) of the standard normal distribution is $\Phi$; $\Phi(x) = \int_{-\infty}^x e^{-\frac{t^2}{2} } \d t$ ($x\in \R$). 
The beta function for $\alpha,\beta \in \Rpp$ is defined as  $B(\alpha ,\beta )=\int_{0}^1 t^{\alpha-1} (1-t)^{\beta-1}\d t$. 

In the sequel, let $\X$ denote a (non-empty) subset of the real line ($\X \subseteq \R$).
A function $\k:\X \times \X \rightarrow \R$ is called kernel if there exists a feature map $\varphi$ from $\X$ to a Hilbert space $\H$ such that $\k(x,y) = \left<\varphi(x),\varphi(y)\right>_{\H}$ for all $x, y\in \X$. While the feature map and the Hilbert space might not be unique, there always exists a unique reproducing kernel Hilbert space (RKHS) $\Hk$ associated to $\k$. $\Hk$ is the Hilbert space of $\X \rightarrow \R$ functions characterized by two properties: $\k(x,\cdot) \in \Hk$ ($\forall x \in \X$) and $f(x)=\left<f,\k(x,\cdot)\right>_{\Hk}$ ($\forall x \in \X$, $f\in\Hk$).\footnote{The shorthand $\k(\cdot, x)$ stands for the function $y \in \X \mapsto \k(y, x) \in \R$ while keeping $x \in \X$ fixed.} The first property describes the basic elements of $\Hk$, the second one is called the reproducing property; combining the two properties makes the canonical feature map and feature space explicit: $\k(x,y) = \left<\varphi_k(x), \varphi_k(y)\right>_{\Hk}$, where $\varphi_k(x) = \k(\cdot, x) \in \Hk$. The closed unit ball of $\Hk$ is denoted by $\Bk = \left\{f\in \Hk \, : \, \left\|f\right\|_{\Hk} \le 1\right\}$.

 Let $\Mp{\X} $ denote the set of Borel probability measures on $\X$. 
 For a given kernel $\k: \X \times \X \rightarrow \R$,  let   
    \begin{equation*} 
    \mu_\k(\P)  =\int_{\X} \k(\cdot,x)\d \P(x) \in \Hk
    \end{equation*}
     denote the mean embedding \citep{berlinet04reproducing,smola07hilbert} of the probability distribution $\P \in \Mp{\X}$; the integral is meant in Bochner sense. 
    The mean embedding is well-defined when $(\P, \k)$
    satisfies
    \begin{aequation}
        \E_{x \sim \P} \sqrt{\k(x,x)} < \infty.  \label{eq:P_k}
    \end{aequation}
  
     The maximum mean discrepancy (MMD) of two distributions $\P, \Q \in \Mp{\X}$  is a semi-metric defined by
     \begin{align}
          &\MMD{\P,\Q}=\left\|\mu_\k(\P) - \mu_\k(\Q) \right\|_{\Hk} \label{eq:MMD-in-terms-of-E} \\
         = &\sup_{f\in \Bk}\left[\E_{x\sim \P} f(x) - \E_{y\sim \Q} f(y)\right] \label{eq:MMD}\\
         =& \sqrt{\left\|\mu_K(\P)\right\|_{\H_K}^2 + \left\|\mu_K(\Q)\right\|_{\H_K}^2  -2 \left<\mu_K(\P),\mu_K(\Q)\right>_{\H_K}} \label{eq:MMDplug}\\
          =& \sqrt{\E_{\substack{x\sim \P,\\x'\sim \P}}\k(x,x') + \E_{\substack{y\sim \Q,\\y'\sim \Q} }\k(y,y') 
          - 2 \E_{\substack{x\sim \P,\\y\sim \Q} }\k(x,y)}, \label{eq:MMDsq}
     \end{align}
        where the second form ($\sup_{f\in \Bk}$) encodes that the discrepancy of two probability distributions is measured by their maximal mean discrepancy over $\Bk$. It also shows that MMD belongs to the class of integral probability metrics \citep{zolotarev83probability,muller97integral}. 
   MMD is well-defined when the pairs $(\P, \k)$ and $(\Q, \k)$ satisfy \eqref{eq:P_k}; this automatically holds for bounded kernels ($\sup_{x\in \X} k(x,x) <\infty$). MMD is a metric if and only if the kernel is characteristic \citep{fukumizu08kernel,sriperumbudur10hilbert}; examples of characteristic kernels include the Gaussian, Laplacian, Mat{\'e}rn, inverse multiquadrics or the B-spline kernel.
             

\section{MMD ESTIMATORS} \label{sec:MMD-estimators}

In this section, we recall existing two-sample MMD estimators in Section~\ref{ssec:classical_mmd}, and we present our proposed semi-explicit ones in Section~\ref{ssec:semi_exp_mmd}. Our motivation for the new estimators is two-fold: (i) to reduce the computational time
, and (ii) to achieve tighter concentration.

\subsection{Classical MMD Estimator}\label{ssec:classical_mmd}


Given i.i.d.\ (independent identically distributed) samples $\{ x_i \}_{i \in [N]}\sim \P$ and $\{ y_i \}_{i \in [M]}\sim \Q$ from the probability measures $\P,\Q \in \Mp{\X}$, one can estimate the squared MMD by using the unbiased U-statistics or the plug-in V-statistics   as
        \begin{align}
           & \MMDshatU{\P_N,\Q_M} = \frac{1}{N(N-1)}\sum_{\substack{i,j\in [N]\\ i\ne j}}\k\left(x_i,x_j\right) \label{eq:MMD-hat-gen:U} \\
            + &\frac{1}{M(M-1)}\sum_{\substack{i,j\in [M]\\ i\ne j}}\k\left(y_i,y_j\right)
            - \frac{2}{NM}  \sum_{\substack{i\in [N]\\ j\in [M]}} \k\left(x_i,y_j\right), \\
            & \MMDshatV{\P_N,\Q_M}  
            = \frac{1}{N^2}\sum_{i,j\in [N]}\k\left(x_i,x_j\right) \label{eq:MMD-hat-gen:V} \\
            + &\frac{1}{M^2}\sum_{i,j\in [M]}\k\left(y_i,y_j\right) - \frac{2}{NM}\sum_{i\in [N]}\sum_{j\in [M]} \k\left(x_i,y_j\right),
        \end{align}
        where $\P_N = \frac{1}{N}  \sum_{n\in [N]} \delta_{x_n}$ and $\Q_M = \frac{1}{M}  \sum_{m\in [M]} \delta_{y_m}$ denote the empirical measures. 
         The estimator 
          $\MMDshatU{\P_N,\Q_M}$ is unbiased, $\MMDshatV{\P_N,\Q_M}$ is non-negative, hence they have complementary advantages; both estimators  have computational complexity $\O\left((N+M)^2\right)$.

\subsection{Proposed Semi-Explicit MMD Estimator}\label{ssec:semi_exp_mmd}

    If the mean embedding $\mu_K(\Q) =  \E_{y\sim \Q}\k(\cdot,y) $ can be computed analytically, one can alternatively estimate the squared MMD using the plugin idea 
    of \eqref{eq:MMD-hat-gen:U}, or that  of \eqref{eq:MMD-hat-gen:V} as 
      \begin{align}
       &   \MMDshatE{\P_N,\Q}= \frac{1}{N(N-1)}\sum_{\substack{i,j\in [N]\\ i\ne j}}\k\left(x_i,x_j\right) \\
        & +\E_{y \sim \Q}  \mu_\k(\Q)(y) -  \frac{2}{N}\sum_{\substack{i\in [N]}} \mu_\k(\Q)(x_i), \label{eq:MMD-hat-gen:exp}\\
         & \MMDshatEV{\P_N,\Q} = \frac{1}{N^2}\sum_{{i,j\in [N]}}\k\left(x_i,x_j\right) \\
         &+  \E_{y \sim \Q}  \mu_\k(\Q)(y)
         -  \frac{2}{N}\sum_{\substack{i\in [N]}} \mu_\k(\Q)(x_i).\label{eq:MMD-hat-gen0:exp} 
     \end{align}
     We will refer to \eqref{eq:MMD-hat-gen:exp} and \eqref{eq:MMD-hat-gen0:exp}  as the semi-explicit MMD estimators; both have computational complexity $\O\left(N^2\right)$. 
     

\section{RESULTS} \label{sec:results}

In this section, we show the theoretical advantage of using explicit mean embedding when available. 
In Section~\ref{ssec:concentration}  we prove tightened concentration results for our semi-explicit MMD estimators,  and extend the analysis to unbounded kernels, with matching minimax lower bounds. We  summarize in Section~\ref{ssec:analytical_meanemb} the kernel-distribution pairs for which we 
derived analytical mean embeddings, extending the current literature. 

\subsection{Concentration of Semi-Explicit MMD} \label{ssec:concentration}
In this section, we show that explicit mean embedding, in case of both bounded and unbounded kernels, leads to better concentration properties of the MMD estimator.   
We start by recalling the concentration of the classical U-statistic based MMD estimator for bounded kernels (Theorem~\ref{th:gretton_concentration}), followed by presenting our tighter result  for the semi-explicit MMD (Theorem~\ref{thm:concentration_explicit}), which we also extend to the unbounded exponential kernel (Theorem~\ref{prop:concentration_exponentiated}) with matching minimax lower bound (Theorem~\ref{th:minimax_unbounded}).

\begin{theorem}[MMD concentration - bounded kernel]
\label{th:gretton_concentration}
Assume that $0\leq \k(x,x') \leq B$ for all $x, x'\in \X$, and let $\epsilon>0$. Then
\begin{align*}
\P\left(\MMDshatU{\P,\Q}-\MMDs{\P,\Q}>\varepsilon\right) \leq 
 e^{-\frac{ \left\lfloor \frac{N}{2} \right\rfloor \varepsilon^2}{8 B^2 } }.
 \end{align*}
The same bound holds for the deviation of $-\varepsilon$ below.
\end{theorem}

Using the analytical knowledge of $\mu_K(\Q)$ leads to tighter concentration  as it is shown by our next result. We recall that $\X \subseteq \R$ ($\X\ne \emptyset$) throughout the manuscript.

\begin{theorem}[Semi-explicit MMD concentration - bounded kernel]
\label{thm:concentration_explicit}
Assume that $A \leq \k(x,x') \leq B$ for all $x,x' \in \X$, and let $\epsilon>0$. Then 
 \begin{align*}
\P\left( \MMDshatE{\P,\Q} -\MMDs{\P,\Q}  >\varepsilon\right)
&\leq  e^{-\frac{  \left\lfloor \frac{N}{2} \right\rfloor \varepsilon^2}{  2 (B-A)^2 } } \\
&+  e^{-\frac{   {N}  \varepsilon^2}{  8 (B-A)^2 } }.
\end{align*}
The same bound holds for the deviation of $-\varepsilon$ below.
\end{theorem}
\noindent\tb{Remarks}: 
\begin{itemize}[labelindent=0em,leftmargin=1em,topsep=0em,partopsep=0cm,parsep=0cm,itemsep=2mm]
     \item 
        The proof of Theorem~\ref{thm:concentration_explicit} relies on rewriting the difference $\MMDshatE{\P,\Q} -\MMDs{\P,\Q}$ as a sum of two U-statistics of different orders, followed by applying twice the Hoeffding inequality for U-statistics and union bounding. 
     \item 
     Specializing Theorem~\ref{thm:concentration_explicit} to $A=0$ and comparing its concentration result with that in Theorem~\ref{th:gretton_concentration} for $\MMDshatU{ \P,\Q}$, we gain in terms of constant in front of $N \varepsilon^2$ in the exponent: we have $\frac{1}{4B^2}$ and $\frac{1}{8B^2}$ instead of $\frac{1}{16B^2}$. This means that the estimator using the analytical knowledge of $\mu_K(\Q)$ brings a factor of $2$ improvement in the \emph{exponent}.
\end{itemize}

In Theorem~\ref{thm:concentration_explicit} the deviation of the estimator $\MMDshatE{\P,\Q}$ was captured for bounded kernels. Our next theorem extends the result to the unbounded exponential kernel, a subcase of the Gaussian-exponentiated  kernel  $(x,y)\mapsto e^{-a(x-y)^2 + b xy}$ with $a=0$. 
\begin{theorem}[Semi-explicit MMD concentration - exponential kernel]
\label{prop:concentration_exponentiated}
Let us consider the exponential kernel $\k(x,y) = e^{b x y}$ ($b>0$, $x,y\in \R$) with probability measures $\P, \Q \in \Mp{\R}$ satisfying 
\begin{align}
\label{eq:unbounded_as}
{ \E_{x \sim \P} e^{ \lambda x^2} < \infty , \quad \E_{x \sim \Q} e^{ \lambda x^2} < \infty \quad \forall \lambda \in \R.}
\end{align}
Let the number of samples $N$ taken from $\P$ be even.  Then for any $p\ge 2$,  there exists a universal constant $C=C_{p,\P,K} >0$ such that for any $\varepsilon>0$
\begin{align*}
\P\left( \MMDshatE{\P,\Q} -\MMDs{\P,\Q}  >\varepsilon\right)
&\leq  \frac{C}{ \varepsilon^p N^{p/2}}.
\end{align*}
The same bound holds for the deviation of $-\varepsilon$ below.
\end{theorem}

\noindent\tb{Remarks:}
\begin{itemize}[labelindent=0em,leftmargin=1em,topsep=0em,partopsep=0cm,parsep=0cm,itemsep=2mm]
\item The proof of Theorem~\ref{prop:concentration_exponentiated} relies on combining concentration results for U-statistics and martingales. One could use similar ideas to cover the two-sample MMD estimator $\MMDshatU{\P,\Q}$ for the exponential kernel.

\item \textbf{Convergence rate of $\MMDshatE{\P,\Q}$:}  
 Theorem~\ref{thm:concentration_explicit} means a convergence rate $\O_{a.s} \left(\frac{1}{\sqrt{N}}\right)$ of the estimator $\MMDshatE{\P,\Q}$ for bounded kernels. Theorem~\ref{prop:concentration_exponentiated} implies the same convergence (when taking $\kappa\rightarrow 0$) for the unbounded exponential kernel. Indeed, for any $\kappa>0$,  one can find $p$ such that $\kappa p >2$. Taking  $\varepsilon_N  = \left(\frac{1}{\sqrt{N}}\right)^{1-\kappa}$ in the Borel-Cantelli lemma, using Theorem~\ref{prop:concentration_exponentiated} and that in this case $\frac{1}{\varepsilon_N^p N^{p/2}} = \frac{N^{\frac{1}{2}(1-\kappa)p}}{N^{p/2}} = N^{-\frac{\kappa p}{2}}$, one arrives at 
\begin{align*}
     &\sum_{N \in \N^*} \P\left( \MMDshatE{\P,\Q} -\MMDs{\P,\Q}  >\varepsilon_N \right) \\
     \leq &\sum_{N \in \N^*} \frac{C_p}{ N^{\frac{\kappa p}{2}}} < \infty.
\end{align*}
\item \textbf{ Convergence of $\MMDshatEV{\P,\Q}$: } Similar  rate can be proved for the V-statistics, by rewriting $\MMDshatEV{\P,\Q} -   \MMDs{\P,\Q}$ in terms of $\MMDshatE{\P,\Q} -   \MMDs{\P,\Q}$ (see the supplement).
\end{itemize}  

It is known \cite[Theorem 2]{tolstikhin16minimax} that the rate $\frac{1}{\sqrt{N}}$ for bounded continuous radial kernels in the two-sample setting for the class of probability measures is optimal with infinitely differentiable density. We prove that a similar result holds for the considered one-sample setting and unbounded exponential kernel. 

\begin{theorem}[Minimax rate for semi-explicit MMD, exponential kernel]
\label{th:minimax_unbounded} Let us consider the exponential kernel $\k(x,y) = e^{b x y}$ ($b>0$, $x,y\in\R$). 
Let $(\P, \k)$ and $(\Q, \k)$  satisfy \eqref{eq:P_k}, and let $m_\P$ and $m_\Q$ stand for the mean of $\P$ and $\Q$, respectively.
Then 
\begin{align*}
&\inf_{ \widehat{\MMDo}_N } \sup_{ \P,\Q \in \mathcal{P}} \P \left( \left| \widehat{\MMDo}_N -  \MMD{\P,\Q} \right| \ge \frac{c}{\sqrt{N}}   \right)\\
\ge& \max\left( \frac{e^{- \frac{a^2 b}{2} }}{4},   \frac{1-\sqrt{  \frac{a^2 b}{2}}}{2}  \right)
\end{align*}
for some finite constant $c>0$, $a=\sqrt{N}(m_\P-m_\Q)$, and  $\widehat{\MMDo}_N$ running over all the estimators using the samples $\{ x_n \}_{n \in [N] }$. 
\end{theorem}
\noindent\tb{Remarks}: 
\begin{itemize}[labelindent=0em,leftmargin=1em,topsep=0em,partopsep=0cm,parsep=0cm,itemsep=2mm]
     \item Theorem~\ref{th:minimax_unbounded} shows that $ \MMD{\P,\Q} $ with $\k$ being the exponential kernel cannot be estimated at a  rate faster than $\frac{1}{\sqrt{N}}$ by any $ \widehat{\text{MMD}}_N $ estimator for all $\P,\Q \in \mP_k$. The fact that the rate $\frac{1}{\sqrt{N}}$ is achievable was shown in Theorem~\ref{prop:concentration_exponentiated}.
    \item The proof relies on the Le Cam's method \citep{lecam73convergence,tsybakov08introduction}. The main technical difference and challenge which were resolved are that using the unbounded exponential kernel one requires a dedicated MMD computation, 
    and with this need the parameter dependence of MMD becomes somewhat intricate.
    \item The condition $\E_{x \sim \P} \sqrt{\k(x,x)} < \infty$ appearing in the definition of $\mP_k$ can only be milder than \eqref{eq:unbounded_as}, since the former is a specific case of \eqref{eq:unbounded_as} with $\lambda = \frac{b}{2}$. In fact, 
    \eqref{eq:unbounded_as} is more restrictive as it can be seen for instance for Gaussian distributions $\P=\mathcal{N}\left(m, \sigma^2\right)$. Indeed, in this case a standard calculation shows that  $\E_{x \sim \P} \sqrt{\k(x,x)}  = e^{\frac{m^2}{1-b\sigma^2}} \frac{1}{ \sqrt{1-\sigma^2 b}}$ which is finite (or equivalently $\P \in \mP_k$) iff $\sigma^2< 1/b$. However,  $\E_{x \sim \P} e^{\lambda x^2}\propto \frac{1}{\sqrt{1-2\lambda \sigma^2}}$ which is finite iff $\lambda < \frac{1}{2 \sigma^2}$; in other words the Gaussian distributions $\mathcal{N}\left(m, \sigma^2\right)$ do not obey \eqref{eq:unbounded_as}.
    
\end{itemize}



\subsection{Analytical Formulas for Mean Embedding}
\label{ssec:analytical_meanemb}
In Section~\ref{ssec:concentration} we showed that one can leverage with the semi-explicit MMD estimator \eqref{eq:MMD-hat-gen:exp} the analytical knowledge of mean embedding, and get tighter concentration. In this section, we provide a summary of our novel results on such closed-form expressions, accompanied with a discussion on existing results.
 
 Our novel analytical mean embedding results (summarized in Table~\ref{tab:results:mean_embedding}, available in Lemma~\ref{lemma:mean_emb_gaussian_gaussian_exp} and Lemma~\ref{lemma:mean_emb_beta_matern}---available in the supplement---and related to existing works below) are on the (Gaussian exponentiated, Gaussian) and (Mat{\'e}rn, beta) kernel-distribution pairs, and are motivated by financial applications. The studied Gaussian-exponentiated and the Mat{\'e}rn kernels generalize the widely-used Gaussian, Laplacian and exponential ones (Fig.~\ref{fig:relation-of-our-kernel-examples}), the beta distribution extends the uniform one  (Fig.~\ref{fig:relation-target-distributions}).
Kernels  ($k$) and distributions ($\Q$) are summarized in Table~\ref{tab:explicit_kernel} and Table~\ref{tab:explicit_density}, with their relations in Fig.~\ref{fig:relation-of-our-kernel-examples} and Fig.~\ref{fig:relation-target-distributions}.

\begin{table}
   \caption{Kernel definitions. Parameters: {$a, b\in \Rnn$}; $\sigma_0,\sigma, \ckg, \lambda \in \Rpp$; $p \in \N$.
   }
  \label{tab:explicit_kernel}
  \begin{center}
\begin{tabular}{@{}l@{\hspace{0.1cm}}l@{}}
    \toprule
       Kernel& $\k(x,y)$   \\
       \midrule
       Gaussian-exponentiated  & $ e^{-a(x-y)^2 + b xy}$  \\
       Mat{\'e}rn & $ \sigma_0^2 e^{-{\frac {{\sqrt {2p+1}}|x-y|}{\sigma }} }{\frac {p!}{(2p)!}} \times$ \\ 
       &  $ \sum _{i=0}^{p}{\frac {(p+i)!}{i!(p-i)!}}\left({\frac {2{\sqrt {2p+1}}|x-y|}{\sigma }}\right)^{p-i}$ \\
       Gaussian & $e^{-a{(x-y)^2}}$\\
      Laplacian & $e^{-\lambda|x-y|}$\\
         exponential  & $e^{bxy}$
       \\ \bottomrule
  \end{tabular}
  \end{center}
\end{table}

\begin{figure}
   \xymatrixcolsep{0.7cm} \xymatrixrowsep{1.2cm}
    \xymatrix{
    \t{Gaussian-exponentiated} \ar[d]_(0.45){a=0} \ar[dr]_(0.35){b=0}  & \t{Mat{\'e}rn} \ar[d]_(.25){p\rightarrow \infty, \, \sigma_0 = 1} \ar[dr]^(0.4){p=0, \, \sigma_0 = 1, \, \sigma= \frac{1}{\lambda}}\\
    \t{exponential} & \t{Gaussian} & \t{Laplacian}
    }
    \caption{Relation of the kernels in Table~\ref{tab:explicit_kernel}. \label{fig:relation-of-our-kernel-examples}}
\end{figure}

\begin{table}
\caption{Target distributions. $q$ stands for the pdf of $\Q$. Parameters: $\cskew, m\in \R$; $\alpha, \beta, v, \sigma \in \Rpp$.}
  \begin{center}
\begin{tabular}{@{}lll@{}}
    \toprule
       Distribution & $q(x)$ \\
       \midrule
       skew Gaussian & $ \frac{2}{\sqrt{2\pi v}} e^{-\frac{(x-m)^2}{2v}} \Phi \left( \frac{\cskew(x-m)}{ \sqrt{v}} \right)$ 
       \\
       Gaussian & $ \frac{1}{ \sqrt{2 \pi \sigma^2}} e^{- \frac{\left( x - m \right)^2}{ 2 \sigma^2 } }$ \\
       beta &  $\frac{1}{B(\alpha,\beta)} x^{\alpha -1}(1-x)^{\beta -1} \I_{[0,1]}(x) $ \\
       uniform & $\I_{[0,1]}(x)$\\
 \bottomrule
  \end{tabular}
  
  \label{tab:explicit_density}
  \end{center}
\end{table}

\begin{figure}
\begin{center}
\begin{tikzcd}[scale=1,row sep=1cm,column sep=8ex,labels=description]  
 \tabnode{skew Gaussian}  \arrow[d,"\tabnode{$s =0$, $v=\sigma^2$}"]
 &\tabnode{beta} \arrow[d,"\text{$\alpha=\beta=1$}"]  \\ 
 \tabnode{Gaussian} &  \tabnode{uniform} 
\end{tikzcd}
\end{center}\vspace{-1mm}
\caption{Relation of the distributions in Table~\ref{tab:explicit_density}. \label{fig:relation-target-distributions}}
\end{figure}

\begin{table}[h] 
 \caption{Summary of obtained analytical mean embedding $\mu_k(\Q)$ results for the considered kernel ($k$) and distribution ($\Q$) pairs.  \label{tab:results:mean_embedding}}
 \begin{center}
\begin{tabular}{llr}
    \toprule
 $\k$ & $\Q$&  $\mu_\k(\Q)$  \\
\midrule 
       Gaussian-exponentiated  & Gaussian& Lemma~\ref{lemma:mean_emb_gaussian_gaussian_exp} \\
       Mat{\'e}rn & beta & Lemma~\ref{lemma:mean_emb_beta_matern}\\
   \bottomrule
\end{tabular}
 \end{center}
 \end{table}

Regarding \tb{related literature}, the analytical expression for the mean embedding of the (Gaussian, Gaussian) kernel-distribution pair \citep{song08tailoring} can be recovered from Lemma~\ref{lemma:mean_emb_gaussian_gaussian_exp} by choosing $b=0$ in the Gaussian-exponentiated kernel, or by taking $\cskew=0$ and $v=\sigma^2$ in the (Gaussian, skew Gaussian) kernel-distribution result \citet[Section 9.2]{kennedy1998bayesian}. Our result on the (Mat{\'e}rn, beta) kernel-distribution pair gives back with $\alpha=\beta=1$ that by \citet[Section 5.4]{briol2019probabilistic} considering the (Mat{\'e}rn, uniform) pair.



\section{EXPERIMENTS}\label{sec:experiments_finance}

In this section, we demonstrate the efficiency of the proposed semi-explicit MMD estimator. 
Our optimization algorithm is presented in Section~\ref{sec:CEM}, followed by the used divergence metrics in Section~\ref{sec:baseline_estimators_cv}.  Our experimental results are presented in Section~\ref{ssec:div_conv}--\ref{sec:parametric-estimation}.

Specifically, we designed the following experiments:
\begin{itemize}[labelindent=0em,leftmargin=1.6em,topsep=0cm,partopsep=0cm,parsep=0cm,itemsep=1mm]
             \item \tb{Experiment 1}: We compare the speed of convergence of different divergence metrics to zero when $\P=\Q$ for Gaussian distributions (Section~\ref{ssec:div_conv}) setting the stage for index replication. 
             \item \tb{Experiment 2}: In Section~\ref{sec:index_replication}, we focus on the index replication problem, and aim to find the index weights matching a target distribution; we tackle the task by the minimization of various divergence measures. 
             
    \item \tb{Experiment 3}: 
     In Section~\ref{sec:parametric-estimation} we focus on MMD estimators 
     to perform the calibration of parametric distributions  on financial data, including LGD ratios and S\&P 500.\footnote{S\&P 500 is a widely used equity index calculated as the weighted-average value of the $500$ most highly capitalised US companies.}
\end{itemize}


\subsection{Optimization Algorithm}\label{sec:CEM} 

 To \tb{optimize} the divergence objectives we tailor the cross-entropy method (CEM; \citealt{rubinstein04cross}) to the task, to address the possible non-convexity of our objective functions.
 Particularly, the CEM technique is a zero-order optimization approach constructing a  sequence of pdfs $f\left(\cdot\,; \bm \theta^{(t)}\right)${---we considered Gaussian distributions---}which gradually concentrates around the optimum as $t\rightarrow \infty$. The idea of the CEM method is to generate samples, followed by adaptively updating $f\left(\cdot\,; \bm \theta^{(t)}\right)$ based on maximum likelihood estimate (MLE) relying on the top $\rho$-percent of the samples (elite in sense of the consider objective function), and smoothing; for further details of the algorithm the reader is referred to the supplement. 

 To deal with the constraints arising in our problems (non-negative and sum to one index weights, non-negative variance) one can apply a softmax transformation on the generated samples from $f\left(\cdot\,; \bm \theta^{(t)}\right)$.
 
%
 
\subsection{Baseline Divergence Estimators in Section~\ref{ssec:div_conv} and \ref{sec:index_replication}}
\label{sec:baseline_estimators_cv}

The performance of our semi-explicit MMD estimator \eqref{eq:MMD-hat-gen0:exp} was compared against the two-sample MMD estimator  \eqref{eq:MMD-hat-gen:V}, and contrasted with estimators of the Wasserstein distance and the Kullback-Leibler (KL) divergence, the latter two being classical baselines \citep[Chapter 5]{lassance19information}.
 
\paragraph{MMD} ($\MMDshatEVsymb$, $\MMDshatVsymb$). We consider  the exponential kernel $\k_{\t{exp}}$ and the Gaussian kernel $\k_{\t{G}}$, for which analytical mean-embeddings can be computed for the Gaussian  distribution (Table~\ref{tab:results:mean_embedding}).
The  $V$-statistic variant of the estimator was taken to guarantee non-negative estimates.

\paragraph{KL Divergence}($\widehat{D}_{\text{KL}}$). In the Gaussian setting, one can rely on the  KL  divergence for Gaussian distributions derived in \citet[page~13]{duchi07derivations}:
for $\P = \mathcal{N}(\mu, \sigma)$, $\Q = \mathcal{N}(m,s)$, which can be evaluated in $\O\left(N\right)$ by computing the empirical mean and variance of samples from $\P$ when $\Q$ is known,
\begin{align*}
D_{\text{KL}} \left(\P, \Q \right)   &= \log \left(\frac{s}{\sigma}\right) + \frac{\sigma^2 + (m-\mu)^2}{2 s^2} - \frac{1}{2}.
\end{align*}

\paragraph{Wasserstein Distance}($\WDhat$). Let $p\ge 1$. The Wasserstein distance \citep{peyre19computational} of the probability measures $\P, \Q \in \Mp{\R}$ is  defined as 
    \begin{align}
            \WD(\P ,\Q ) 
             & = \left( \int_0^1 \left| F_\P^{-1}(t) - F_\Q^{-1}(t) \right|^p \d t\right)^{1/p}\label{eq:WAD},
 \end{align}
 where 
 $F^{-1}_\P$ and $F^{-1}_\Q$ are the inverse cdfs of $\P$ and $\Q$, and $L^p([0,1])$ refers to the real-valued $p$-power Lebesgue-integrable functions on $[0,1]$. 
An empirical estimator for $\WD$ is, for $\{ x_i \}_{i \in [N]}\sim \P$: 
\begin{align}
 \WDhat(\P_{N},\Q) = \left[ \frac{1}{ N }\sum_{j \in \left[ N \right]} \left| x_{ \left( j \right)} -  F^{-1}_{\Q}\left( \frac{j}{N} \right) \right|^p  \right]^{\frac{1}{p}}.
\end{align}
This estimator can be evaluated in $\O\left(N\log(N)\right)$ time.

\subsection{Speed of Convergence}\label{ssec:div_conv}

In this section, we explore the speed of convergence of the divergence measures detailed in Section~\ref{sec:baseline_estimators_cv}.

When two probability distributions $\P$ and $\Q$ are equal, a desirable property of their  divergence estimator is to converge quickly towards zero; this is what we investigate next. Particularly, 
we assess the convergence of $D\left({\Q}_{N}, \Q\right)$, $\Q\sim \cN(m,s)$\footnote{The values of $m$ and $s$ were chosen as in the index replication setting (Section~\ref{sec:index_replication}): $m=0.042$ and $s=0.0719$.}, based on samples $\{x_i\}_{i\in[N]} \sim \Q$ for varying $N$. We took $b=10^{-3}$ for $\k_{\t{exp}}$, $c=2$ for $\k_{\t{G}}$ and for  $\WDhat$  we chose $p=1$ (but got similar results for other values of $p$). For the two-sample MMD estimator, we took $M=N$.  For each fixed sample size, we performed $100$ Monte Carlo simulations to assess the variability of the estimation. The obtained mean and std results are summarized in Fig.~\ref{fig:errorbar}. As it can be seen in the figure, the semi-explicit MMD estimator converges faster than the two-sample MMD one to $0$, with lower std. The exponential kernel provides lower values of the divergence. All divergence metrics show a convergence rate of $\O\left(1/\sqrt{N}\right)$ except for the Wasserstein metric whose slope is around $-0.3$ in log-log scale for the considered samples. 

\begin{figure*}[h]
\begin{center}
{\includegraphics[scale=0.4]{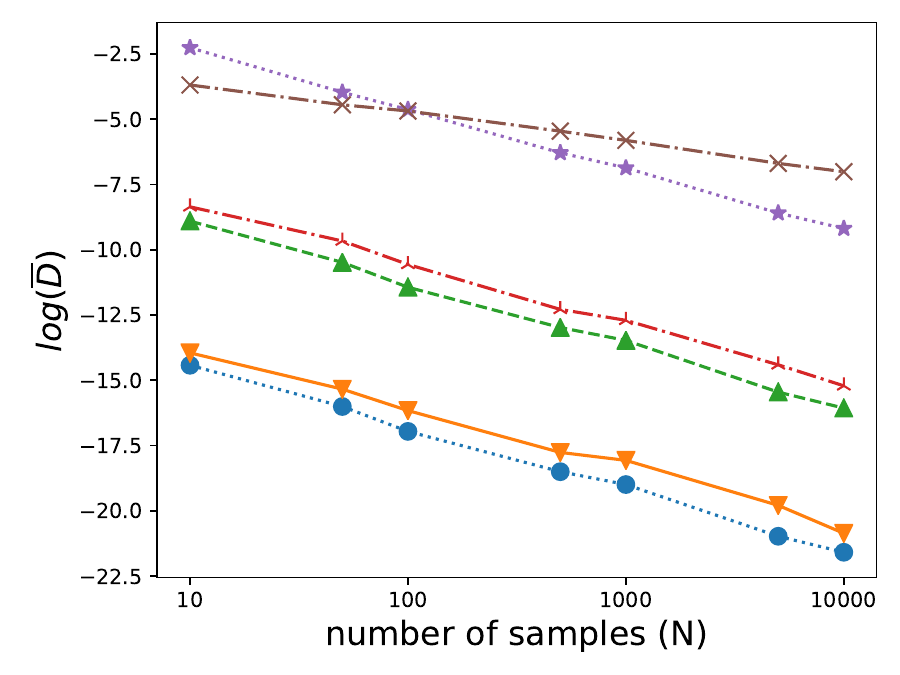}}
{\includegraphics[scale=0.4]{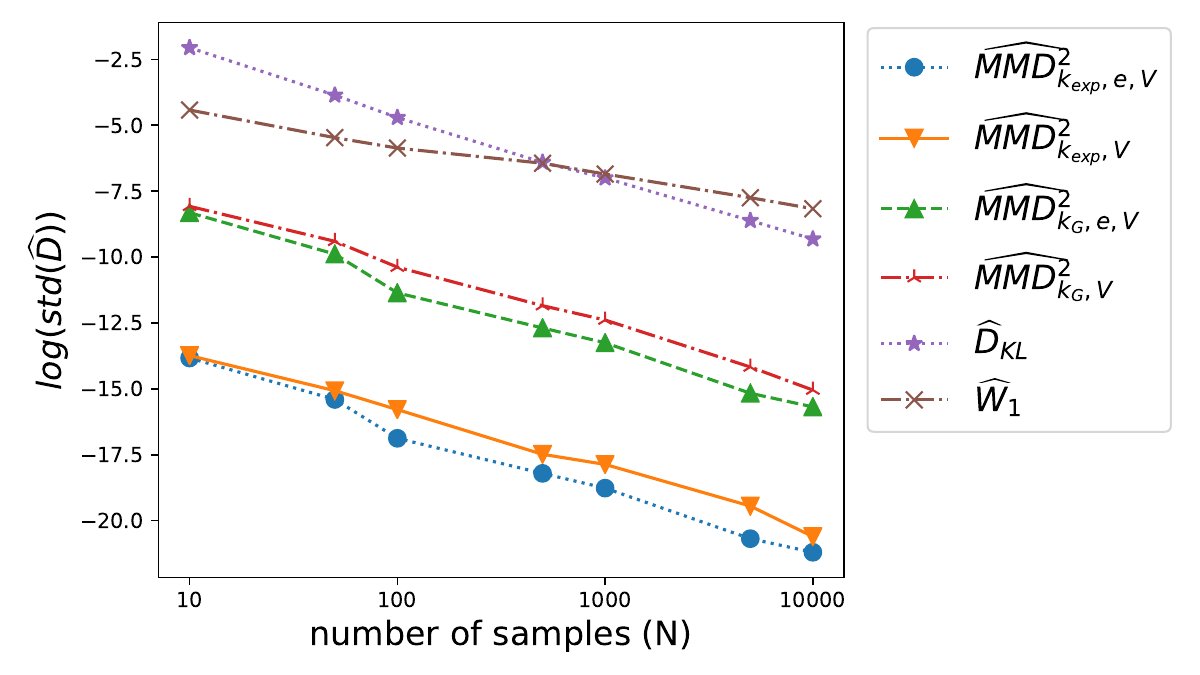}}
\caption{Mean $\pm$ std of various divergences when $\P = \Q$, on log-log scale as a function of number of sample $N$.}\label{fig:errorbar}
\end{center}
\end{figure*}

\subsection{Index Replication} \label{sec:index_replication}

In the index replication problem, the aim is to find the weights ($\b w$) for a basket of stocks based on the knowledge of the distributions of the stocks returns ($\b r$) and that of the index ($\P_T$). The problem can be formulated by  minimizing the discrepancy (measured in the sense of a divergence $D$) between the associated distributions $\bw^{\top} \br \sim \P_{\b w} $ and $\P_T$:
\begin{align}
     \bw ^* & = \argmin_{\bw \in \W^d} D\left(\P_\bw,\P_T\right),
     \label{eq:target-distr-approach}
\end{align}
where $\W^d$ encodes non-negative weights summing to one.

In this experiment, we designed a challenging low signal-to-noise setting to replicate what is observed in financial markets. Particularly, we worked with $d=3$ and assumed $\P_T=\mathcal{N}\left(m,s \right)$ with $m= \left(\bw^0\right)^\top \bm{\mu}$, $ s^2 =\left(\bw^0 \right)^\top [\Diag(\bm{\sigma})]^2 \left(\bw^0\right)$,  $\bw^0 = (0.7,0.2,0.1)$, $\bm{\mu}=(0.05,0.03,0.01)$ and $\bm{\sigma}=(0.1,0.08,0.05)$ and consequently $m=0.042$ and $s=0.0719$, with a std approximately twice the mean. 

We used  $N=1500$ (corresponding to $5$ years of data) to generate $\{\br\}_{i\in[N]} \sim \mathcal{N}\left(\bm{\mu} ,\Diag(\bm{\sigma})\right)$,  and applied the CEM algorithm to solve
\begin{align}
     \bw ^* & = \argmin_{\bw \in \W^d} D\left(\P_{\bw,N},\P_T\right),
\end{align}
which is the empirical counterpart of \eqref{eq:target-distr-approach}. 
We considered all the divergence measures and estimators detailed in Section~\ref{sec:baseline_estimators_cv}; for MMD only the semi-explicit estimator was taken due to its faster convergence experienced in Section~\ref{ssec:div_conv}.

Our results, summarized in Fig.~\ref{fig:errorbarw}, show the  convergence towards $\bw^0$ for the different estimators. In this example the performance of the Gaussian kernel based MMD and the KL divergence supersede that of the exponential kernel based MMD; the latter shows a larger variability which could be explained by the unbounded nature of the exponential  kernel.

\begin{figure*}[h]
\begin{center}
{\includegraphics[trim={0 0 5.7cm 0},clip,scale=0.4]{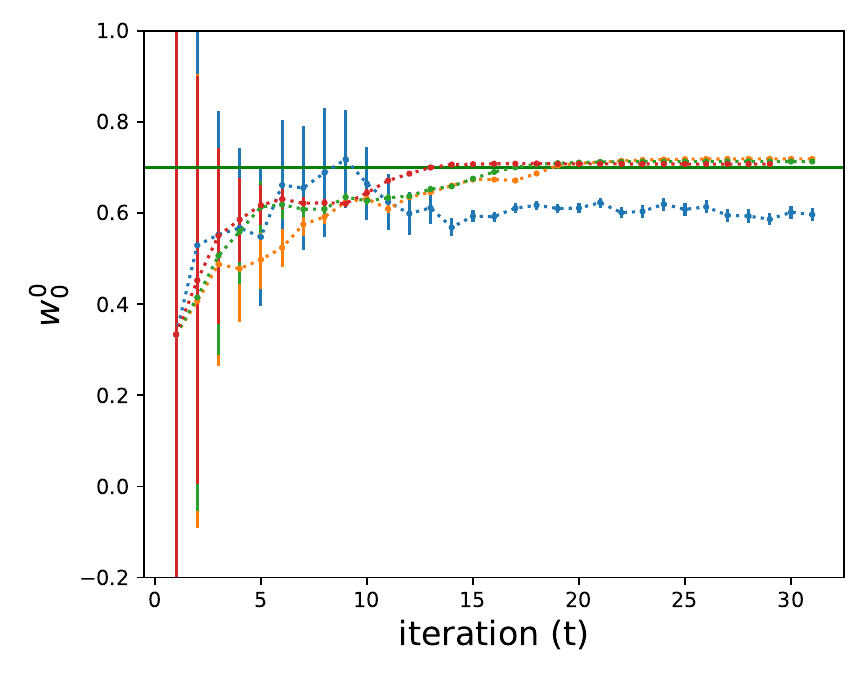}}
{\includegraphics[scale=0.4]{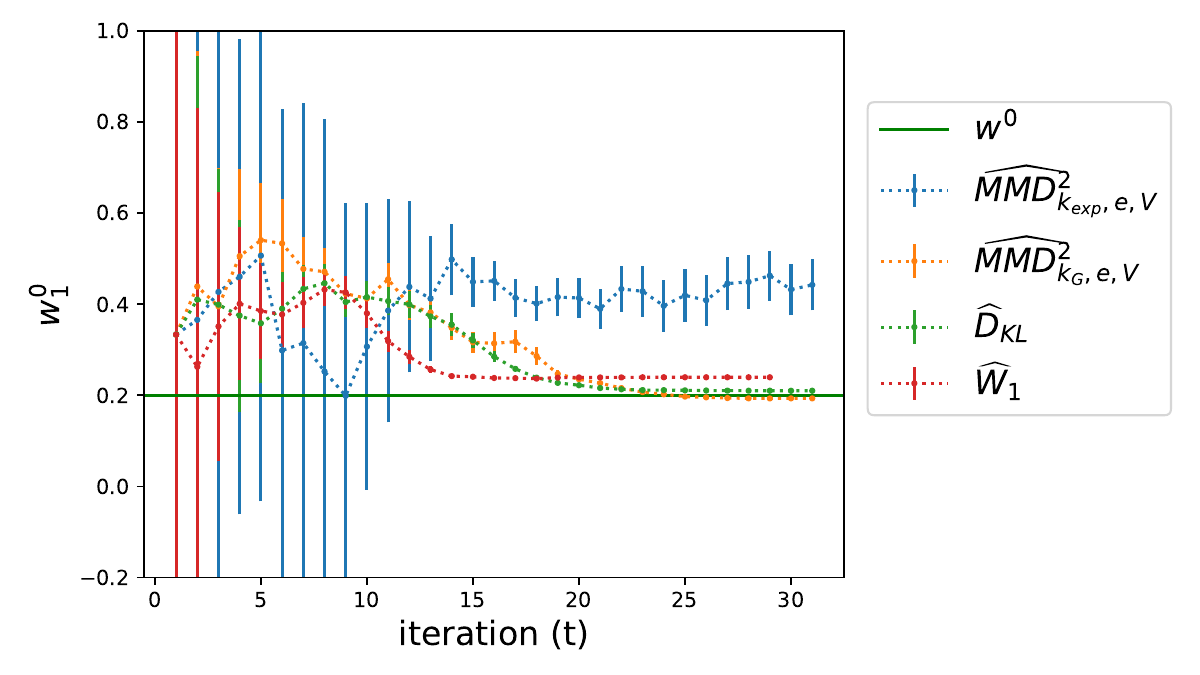}}
\caption{{Estimated weights (first two coordinates of $\b w^0$; $w^0_3 = 1- (w^0_1+w^0_2)$) in index replication as a function of number of iterations}.}\label{fig:errorbarw}
\end{center}
\end{figure*}

\subsection{Parametric Estimation on Financial Data} \label{sec:parametric-estimation}
Our two sub-experiments for parametric estimation on financial data using MMD were as follows.
\begin{enumerate}
    \item Calibration of beta distribution on a  dataset of historical LGD rates provided by a European bank.\footnote{The data is available at \href{http://www.creditriskanalytics.net/datasets-private2.html}{Credit Risk Analytics}.} It includes $N=2545$ observations on  LGDs. 
    \item Calibration on the S\&P 500\footnote{Data can be found on \href{https://finance.yahoo.com/}{Yahoo Finance}.} of a  
    Gaussian distribution, and of a skew Gaussian distribution.\footnote{Although  Gaussian is a special case of skew Gaussian, we will compare the goodness-of-fit in the two cases.} 
\end{enumerate}
In line with these calibration tasks we considered the parametric family of distributions $\{\Q_{\theta}\,:\, \theta\in \Theta\}$ as beta, Gaussian or skew Gaussian, and investigated the benefits of the semi-explicit MMD estimator compared to the two-sample one.

Particularly, given samples  $\{x_i\}_{i\in [N]}$ used for calibration and noticing that the terms in \eqref{eq:MMD-hat-gen:U} and \eqref{eq:MMD-hat-gen:exp}  containing $\k(x_i,x_j)$ are independent from $\theta$, 
the associated problems boil down to the optimization tasks:
\begin{align}
\theta^*&:=\argmin_{\theta \in \Theta} \MMDshatE{\P_N, \Q_{\theta} }  \\
&=\argmin_{\theta \in \Theta} \E_{y \sim  \Q_{\theta} }\left( \mu_{\k,\theta}(y) \right) -  \frac{2}{N}\sum_{\substack{i\in [N]}} \mu_{\k}(\Q_\theta)(x_i),\\
\theta^*&:=\argmin_{\theta \in \Theta}\ 
\MMDshatU{\P_N,\Q_{\theta,M}}\\
&= \argmin_{\theta \in \Theta} \frac{1}{M(M-1)}\sum_{\substack{i,j\in [M]\\ i\ne j}}\k\left(y_i,y_j\right) \\
 & \hspace{1.72cm}- \frac{2}{NM}\sum_{\substack{i\in [N] \\ j\in [M]}} \k\left(x_i,y_j\right).
\end{align}
where the dependence in the latter objective (two-sample MMD estimator) in $\theta$ is via the sample $\{y_i \}_{i\in[M]}\stackrel{\t{i.i.d.}}{\sim} \Q_{\theta}$.


We chose the following (kernel, distribution) pairs for which analytical mean-embedding  are available:
\begin{itemize}
    \item for the calibration of LGD ratios we selected $(\k_{\t{L}},\Q_{\t{b}}):=$ (Laplacian, beta),
    \item for S\&P 500 we used the pairs:
    $(\k_{\t{exp}},\Q_{\t{G}}):=$  (exponential, Gaussian),
    $(\k_{\t{G}},\Q_{\t{sG}}):=$  (Gaussian, skew Gaussian).
\end{itemize}

In Table~\ref{table:results} we report  the mean value and the standard-deviation of the objective function in the semi-explicit and two-sample case.  We also performed a test of distribution adequacy using the approximated p-value. 
\begin{table}[h]
\caption{Parametric estimation results (mean $\pm$ std) computed on the final elite samples of the CEM algorithm. The p-value is for the semi-explicit estimator.} \label{table:results}
\begin{center}
\begin{tabular}{@{}l@{\hspace{0.1cm}}r@{\hspace{0.1cm}}c@{\hspace{0.1cm}}r@{\hspace{0.1cm}}r@{\hspace{0.1cm}}c@{\hspace{0.1cm}}r@{\hspace{0.2cm}}c@{}}
    \toprule
$(\k,\Q)$& \multicolumn{3}{c}{$\MMDshatEsymb$}  & \multicolumn{3}{c}{$\MMDshatUsymb$}  &{$\hat{p}$} \\
\midrule
$(\k_{\t{L}},\Q_{\t{b}})$  & $-10^{-3}$ & $\pm$ & $5\cdot10^{-4}$ & $0.5$ & $\pm$ & $8\cdot 10^{-2}$ &$0.63$ \\
$(\k_{\t{exp}},\Q_{\t{G}})$ & $3 \cdot10^{-6}$ & $\pm$ & $8\cdot10^{-10}$ & $4\cdot10^{-5}$ & $\pm$ & $2\cdot10^{-6}$  &$0.66$ \\
$(\k_{\t{G}},\Q_{\t{sG}})$ & $5 \cdot 10^{-3}$ & $\pm$ & $2 \cdot 10^{-6}$ & $-2 \cdot 10^{-3}$ & $\pm$ & $7 \cdot 10^{-6}$  & $0.95$  \\
 \bottomrule
\end{tabular}
\end{center}
\end{table}

As it can be seen in Table~\ref{table:results}, the semi-explicit MMD consistently provides estimates with a significantly smaller std in the last iteration of the CEM algorithm. It is worth noting that the positive skewness on the last $5$ years of S\&P 500 returns is correctly captured, leading to a higher p-value with a skew Gaussian rather than with a Gaussian target.

\onecolumn
\appendix
\aistatstitle{Supplementary Material}

In the supplement, we provide the proofs our concentration (Section~\ref{sec:one}) and analytical mean embedding    (Section~\ref{sec:two}) results. External statements are collected in Section~\ref{sec:external-statements}. 
Further details about our experiments are given in Section~\ref{sec:CEMsupp}.

\section{PROOFS}\label{sec:one}



In this section,  we present the proofs of our concentration results for the semi-explicit MMD estimators. In Section~\ref{ssec:proof_conc_results}, we give the detailed proofs of our  tightened concentration results for the unbiased semi-explicit MMD estimator for bounded kernels (Theorem~\ref{thm:concentration_explicit}) and for  the unbounded exponential kernel  (Theorem~\ref{prop:concentration_exponentiated}), both leading to a $1/\sqrt{N}$ convergence rate. We follow by extending the convergence rate of the unbiased semi-explicit MMD to its V-statistics counterpart in Section~\ref{ssec:cv_VMMD}. We then (Section~\ref{ssec:MMD_minimax}) provide the proof of the optimality of the rate $1/\sqrt{N}$ in the unbounded case of exponential kernel (Theorem~\ref{th:minimax_unbounded}).

\subsection{Proofs of Concentration Results}\label{ssec:proof_conc_results}

\begin{proof}{(Theorem~\ref{thm:concentration_explicit}; concentration of $\MMDshatEU{\P,\Q}$, bounded kernel)}
By the definition of $\MMDshatEU{\P,\Q}$ and $\MMDs{\P,\Q} $, they have the term $\E_{y \sim \Q}  \mu_\kk(\Q)(y)$ in common, hence their difference writes as
 \begin{align}
 & \quad \MMDshatEU{\P,\Q} -\MMDs{\P,\Q} \\ 
  &\stackrel{(a)}{=} \underbrace{\frac{1}{N(N-1)} \sum_{\substack{i,j\in [N]\\ i\ne j}} \kk(x_i,x_j)}_{=:U_2} -  \underbrace{\E_{x, x'\sim \P}\kk(x,x')}_{\E U_2 }  -  2\Bigg[ \underbrace{\frac{1}{N} \sum_{i \in [N]} \mu_\kk(\Q)(x_i)}_{=:U_1} - \underbrace{\E_{x\sim \P} \mu_\kk(\Q)(x)}_{\E U_1} \Bigg] \nonumber\\
   & = U_2-\E U_2 - 2(U_1-\E U_1 )\label{eq:MMD_diff_U}
 \end{align}
by using in (a)  that $U_1$ and $U_2$ are U-statistics. The kernel of $U_1$ is $h_1(x) = \mu_\kk(\Q)(x)$, the kernel of $U_2$ is  $h_2(x,x')=k(x,x')$.  Since by assumption the kernel $k$ is lower bounded by $A$ and upper bounded by $B$, the same property holds for $h_1$ and $h_2$. Hence applying the Hoeffding bound for U-statistics (Theorem~\ref{th:Hoeffding_thm}), for any $t>0$
\begin{align}
  \P\left( U_1-\E{U_1} < - t \right) &=\P\left( (-U_1)-\E{(-U_1)} >  t \right)\le e^{-\frac{  2 N t^2}{  (B-A)^2 } },\\
     \P\left( U_2-\E{U_2} >  t \right) &\le e^{-\frac{  2\left\lfloor \frac{N}{2} \right\rfloor t^2}{ (B-A)^2 } }. \label{eq:U-stat-bounds}
\end{align}
Returning to our target quantity $\MMDshatEU{\P,\Q} -\MMDs{\P,\Q}$, for any $\varepsilon>0$
\begin{align}\hspace{-1cm}
 \left\{ \MMDshatEU{\P,\Q} -\MMDs{\P,\Q}  >\varepsilon \right\} 
      & \stackrel{\eqref{eq:MMD_diff_U}}{=} \left\{  U_2-\E{U_2} - 2(U_1-\E{U_1}) >\varepsilon  \right\}\nonumber\\
      &\stackrel{(a)}{\subseteq} \left\{ U_2-\E{U_2} >  \frac{\varepsilon}{2} \right\} \cup \left\{U_1-\E{U_1} < - \frac{\varepsilon}{4} \right\}, \label{eq:decomp}
\end{align}
where the inclusion $A \subseteq B \cup C$ in (a) is equivalent to $\bar{B} \cap \bar{C} \subseteq \bar{A}$; the latter holds as $\left\{ U_2-\E{U_2} \le  \frac{\varepsilon}{2} \right\} \cap \left\{- 2(U_1-\E{U_1}) \le  \frac{\varepsilon}{2} \right\} \subseteq \left\{U_2-\E{U_2} - 2(U_1-\E{U_1}) \le \varepsilon \right\}$. Using \eqref{eq:decomp} and the bound \eqref{eq:U-stat-bounds}  with $t =\frac{\epsilon}{2}$ for $U_2$ and $t=\frac{\epsilon}{4}$ for $U_1$ one arrives at
\newpage
\begin{align*}
\P\left(  \MMDshatEU{\P,\Q} -\MMDs{\P,\Q}  >\varepsilon \right)
 & \leq  \P\left( U_2-\E{U_2} >  \frac{\varepsilon}{2} \right)+ \P\left( U_1-\E{U_1} < -  \frac{\varepsilon}{4} \right) \\
& \leq e^{-\frac{  2\left\lfloor \frac{N}{2} \right\rfloor \varepsilon^2}{  2^2 (B-A)^2 } } +  e^{-\frac{  2 {N}  \varepsilon^2}{  4^2 (B-A)^2 } } = e^{-\frac{ \left\lfloor \frac{N}{2} \right\rfloor \varepsilon^2}{  2 (B-A)^2 } } +  e^{-\frac{   {N}  \varepsilon^2}{  8 (B-A)^2 } }.
\end{align*}

To establish the bound in $<-\varepsilon$, we can use a similar union bounding argument as in \eqref{eq:decomp}:
\begin{align}
    & \quad  \left\{ \MMDshatEU{\P,\Q} -\MMDs{\P,\Q}  <-\varepsilon \right\}  = \left\{  U_2-\E{U_2} - 2(U_1-\E{U_1}) <-\varepsilon  \right\}\nonumber\\
      &\stackrel{(b)}{\subseteq} \left\{ U_2-\E{U_2} <- \frac{\varepsilon}{2} \right\} \cup \left\{- 2(U_1-\E{U_1}) <- \frac{\varepsilon}{2} \right\} \label{eq:decomp2} \\
            & = \left\{ -U_2-\E{-U_2} > \frac{\varepsilon}{2} \right\} \cup \left\{- (U_1-\E{U_1}) <- \frac{\varepsilon}{4} \right\}, \nonumber
\end{align}
where the inclusion $A \subseteq B \cup C$ in (b) is equivalent to $\bar{B} \cap \bar{C} \subseteq \bar{A}$; the latter holds as $\left\{ U_2-\E{U_2} \ge -  \frac{\varepsilon}{2} \right\} \cap \left\{- 2(U_1-\E{U_1}) \ge - \frac{\varepsilon}{2} \right\} \subseteq \left\{U_2-\E{U_2} - 2(U_1-\E{U_1}) \ge- \varepsilon \right\}$.
Using \eqref{eq:decomp2} and the bound \eqref{eq:U-stat-bounds} replacing $U_1$ by $-U_1$ and $U_2$ by $-U_2$,  with $t =\frac{\epsilon}{2}$ for $U_2$ and $t=\frac{\epsilon}{4}$ for $U_1$ we arrived at
\begin{align*}
\P\left(  \MMDshatEU{\P,\Q} -\MMDs{\P,\Q}  >\varepsilon \right)  &\leq  \P\left( (-U_2)-\E{(-U_2)} >  \frac{\varepsilon}{2} \right)+ \P\left( (-U_1)-\E{(-U_1)} < -  \frac{\varepsilon}{4} \right) \\
& \leq e^{-\frac{  2\left\lfloor \frac{N}{2} \right\rfloor \varepsilon^2}{  2^2 (B-A)^2 } } +  e^{-\frac{  2 {N}  \varepsilon^2}{  4^2 (B-A)^2 } } = e^{-\frac{ \left\lfloor \frac{N}{2} \right\rfloor \varepsilon^2}{  2 (B-A)^2 } } +  e^{-\frac{   {N}  \varepsilon^2}{  8 (B-A)^2 } }.
\end{align*}
\end{proof}

\begin{proof}{(Theorem~\ref{prop:concentration_exponentiated}; concentration of $\MMDshatEU{\P,\Q}$, exponential kernel)}

$\MMDs{\P,\Q}$ is well-defined since $\E_{x \sim \P} \sqrt{\kk(x,x)} = \E_{x \sim \P} e^{\frac{bx^2}{2}} $ and  $\E_{x \sim \Q} \sqrt{\kk(x,x)}  = \E_{x \sim \Q} e^{\frac{bx^2}{2}}$ are finite by assumption \eqref{eq:unbounded_as}.

Similarly to the proof of Theorem~\ref{thm:concentration_explicit}, we write the difference   $\MMDshatEU{\P,\Q}-\MMDs{\P,\Q} $ in terms of two $U$-statistics 
 \begin{eqnarray*}
 \lefteqn{\MMDshatEU{\P,\Q} -\MMDs{\P,\Q}}\\
 &&= \underbrace{\frac{1}{N(N-1)} \sum_{\substack{i,j\in [N]\\ i\ne j}} \kk(x_i,x_j)}_{=:T_2} -  \underbrace{\E_{x, x'\sim \P}\kk(x,x')}_{\E T_2 }  -  2\Bigg[ \underbrace{\frac{1}{N} \sum_{i \in [N]} \mu_\kk(\Q)(x_i)}_{=:T_1} - \underbrace{\E_{x\sim \P} \mu_\kk(\Q)(x)}_{\E T_1} \Bigg] \\
   && = T_2-\E T_2 - 2(T_1-\E T_1 )
\end{eqnarray*} 
with $T_1$ having the kernel $h_1(x) = \mu_\kk(\Q)(x)$ and $T_2$ using the kernel $h_2(x,x')=k(x,x')$. 
We establish a concentration result on $T_1 - \E T_1$ and $T_2 - \E T_2$ separately, and combine them with the union bound:
\begin{align}
& \quad      \left\{ \MMDshatEU{\P,\Q} -\MMDs{\P,\Q}  >\varepsilon \right\}  = \left\{ T_2-\E{T_2} - 2(T_1-\E{T_1}) >\varepsilon  \right\}\nonumber\\
      &\subseteq \left\{ T_2-\E{T_2} > \frac{\varepsilon}{2} \right\} \cup \left\{- 2(T_1-\E{T_1}) > \frac{\varepsilon}{2} \right\}= \left\{ T_2-\E{T_2} > \frac{\varepsilon}{2} \right\} \cup \left\{T_1-\E{T_1} <- \frac{\varepsilon}{4} \right\}.  \label{eq:decomp_u}
\end{align}
Below let $p\ge 2$ denote a fixed constant. By assumption $\Nt := \frac{N}{2} \in \N^{*}$.
\begin{itemize}[labelindent=0em,leftmargin=1em,topsep=0em,partopsep=0cm,parsep=0cm,itemsep=2mm]
\item \textbf{Bound on $T_2$}: Let us introduce the notation $V$ for the sum of independent processes 
\begin{align*}
V(x_1, x_2,\dots ,x_N) & =   \kk(x_1,x_2)+ \kk(x_3,x_4)+ \dots +  \kk(x_{N-1},x_{N}).
\end{align*}
With this notation our target quantity $T_2$ can be rewritten \citep{pitcan2017note} as
\begin{align}
T_2 & = \frac{1}{N(N-1)} \sum_{\substack{ i,j \in [N] \\ i \ne j}} \kk(x_i,x_j) =\frac{2}{N} \left( \frac{1}{N!} \sum_{\sigma \in S_N}  V(x_{\sigma_1}, \dots ,x_{\sigma_N})\right), \label{eq:perm_U_stat}
\end{align} 
where $S_N$ denotes the set of permutations of $[N]$.  Then  for any $t>0$
 \begin{align}
   \lefteqn{\P\left( T_2 - \E T_2 > t  \right)  = \P\left(\frac{1}{N(N-1)} \sum_{i \ne j} \kk(x_i,x_j) - \E_{x,x' \sim \P }\kk(x,x') > t  \right)}  \nonumber\\\
   & \stackrel{(a)}{\le} \frac{\E{ \left|\frac{1}{N(N-1)} \sum_{i \ne j} \kk(x_i,x_j) - \E_{x,x' \sim \P }\kk(x,x')  \right|^p }}{t^p}
\stackrel{(b)}{=}   \frac{ \E{ \left\lvert\frac{2}{N} \frac{1}{N!} \sum_{\sigma \in S_N}  V(x_{\sigma_1}, \dots ,x_{\sigma_N}) - \E_{x,x' \sim \P }\kk(x,x')  \right\rvert^p }}{t^p}\nonumber\\
& \stackrel{(c)}{=}   \left( \frac{2}{Nt} \right)^p  \E{ \left\lvert \frac{1}{N!} \sum_{\sigma \in S_N} \left[  V(x_{\sigma_1}, \dots ,x_{\sigma_N}) -\frac{N}{2} \E_{x,x' \sim \P }\kk(x,x') \right]  \right\rvert^p }\nonumber\\
& \stackrel{(d)}{\leq} \left( \frac{2}{N t} \right)^p \frac{1}{N!} \sum_{\sigma \in S_N} \E{\Big|\underbrace{ V(x_{\sigma_1}, \dots ,x_{\sigma_N}) - \frac{N}{2} \E_{x,x' \sim \P }\kk(x,x') }_{  M^\sigma_{ \Nt}  } \Big|^p}.
\label{eq:U_trick}
\end{align} 
(a) comes from the  generalized Markov's inequality (Lemma \ref{lemma:gen_markov_ineq}) by choosing $\phi(x) :=|x|^p$ and $I=\R$.
In (b) we applied \eqref{eq:perm_U_stat}. Pulling out $\left(\frac{2}{N}\right)^p$ gives (c).  
(d) follows from the Jensen inequality by applying it to the argument of the expectation with the convex function $x \mapsto |x|^p$.

Let us introduce the notation $M^\sigma_{  \Nt} = V(x_{\sigma_1}, \dots ,x_{\sigma_N}) -\frac{N}{2} \E_{x,x' \sim \P }\kk(x,x') $ in \eqref{eq:U_trick}. One can expand $ M_{ \Nt} $ as a sum of centered independent processes:
\begin{align*}
M^\sigma_{ \Nt}& = \underbrace{
 \kk(x_{\sigma_1},x_{\sigma_2})+ \kk(x_{\sigma_3},x_{\sigma_4})+ \dots +  \kk(x_{\sigma_{N-1}},x_{\sigma_N}) }_{\Nt= N/2 \text{  terms}}- \Nt \E_{x,x' \sim \P }\kk(x,x')
\\
&=
 \sum_{k\in\left[\Nt\right]}\underbrace{ \left[\kk(x_{\sigma_{2k-1}},x_{\sigma_{2k}}) - \E_{x,x' \sim \P }\kk(x,x')\right]}_{=:Y_k}.
\end{align*}
Similarly, let us denote $M^\sigma_n = \sum_{k \in [n]} Y_k, \ n  \in \left[ \Nt \right]$.
By definition $\E Y_k = 0$ for all $k\in [n]$, which implies that $M^\sigma_{n}$ is a martingale w.r.t.\ the filtration $\mathscr{F}_n=\sigma\left((Y_k)_{k\in [n]}\right)$:
\begin{align*}
\E\left(  M^\sigma_{n} \lvert \mathscr{F}_{n-1} \right)=\E(   \underbrace{ M^\sigma_{ n -1 } }_{  \mathscr{F}_{n-1}-\text{measurable}} + \underbrace{Y_{n}}_{  \text{independent from } \mathscr{F}_{n-1}} \lvert\mathscr{F}_{n-1} ) = M^\sigma_{ n -1 } + \underbrace{\E( Y_{n})}_{0}  = M^\sigma_{ n-1 }.
\end{align*}
Hence, we can apply the Burkholder's inequality (Theorem \ref{th:burkholder}) on the martingale $\left\{(M^\sigma_n,\mathscr{F}_n)\right\}_{n\in \Nt}$; it ensures the existence of a constant $C_p>0$ such that 
\begin{align}
     \E \left| M^\sigma_{\Nt} \right\rvert^p &\leq C_p \E \left( \sum_{k \in \Nt} Y_k^2 \right)^{p/2}
     = C_p \left(\Nt\right)^{p/2}\E \left(\frac{1}{\Nt} \sum_{k \in \Nt} Y_k^2 \right)^{p/2}   \stackrel{(a)}{\le} C_p \left(\Nt\right)^{p/2}\hspace*{-1cm}\underbrace{\E \left( \frac{1}{\Nt} \sum_{k \in \left[\Nt\right]} \rvert Y_k \rvert^p \right)}_{=:m_p=\E_{x,x' \sim \P } \left| \kk(x,x') - \E_{x,x' \sim \P } \kk(x,x') \right|^p}\hspace*{-1cm},
     \label{eq:m_p-def}
\end{align}
where in (a) we applied the Jensen inequality to the argument of the expectation with the convex function $x \mapsto x^{p/2}$. Moreover, $m_p$ is finite since by assumption \eqref{eq:unbounded_as} with $\lambda = bp$, $\E_{x \sim \P } \lvert\kk(x,x) \rvert^p  = \E_{x\sim \P} e^{p b x^2} < \infty$ and $\E_{x,x' \sim \P } \left| \kk(x,x')\right|^p \leq \sqrt{\E_{x \sim \P } \lvert\kk(x,x) \rvert^p \E_{x' \sim \P } \lvert\kk(x',x') \rvert^p}$. 
By \eqref{eq:m_p-def} we arrive at
\begin{align*}
\E{ \left\lvert \sum_{k\in \Nt}  \left[ \kk(x_{\sigma_k},x_{\sigma_{k+1}}) -   \E_{x,x' \sim \P }\kk(x,x')\right] \right\rvert^p  } =\E \left\lvert M^\sigma_{\Nt} \right\rvert^p & \leq
C_p m_p  \left( \Nt \right)^{p/2}
\end{align*}
which combined with \eqref{eq:U_trick} gives that for any $t>0$
\begin{align}
\P\left( T_2 - \E T_2 > t  \right) &  \leq\left( \frac{2}{N t} \right)^p \frac{1}{N!} \sum_{\sigma \in S_N} C_p m_p  \left( \frac{N}{2} \right)^{p/2} = C_p m_p \left( \frac{2}{N t^2} \right)^{p/2}.\label{eq:bound_T2}
\end{align}
Note: The same bound holds for $-t$
\begin{align}
\P\left( T_2 - \E T_2 <- t  \right) &  \leq C_p m_p \left( \frac{2}{N t^2} \right)^{p/2}, \label{eq:bound_T2b}
\end{align}
 by changing $\kk$ to $-\kk$ in all the previous steps.
  
\item \tb{Bound on $T_1$}: Applying the generalized Markov's inequality (Lemma~\ref{lemma:gen_markov_ineq}) with $\phi(x) :=|x|^p$, one can bound the probability $\P(T_1 - \E T_1 >t)$ in terms of $\E|T_1 - \E T_1|^p$ for any $t>0$ as
\begin{align}
\P(T_1 - \E T_1 >t) &= \P\left( \frac{1}{N} \sum_{n \in [N]}  \mu_\kk(\Q)(x_n)-\E_{x\sim \P} \mu_\kk(\Q)(x) >t \right)   \leq \frac{\E \left\lvert  \frac{1}{N} \sum_{n \in [N]}  \mu_\kk(\Q)(x_n)-\E_{x\sim \P} \mu_\kk(\Q)(x) \right\rvert^p}{t^p} \nonumber\\
& = \frac{1}{(N t)^p} \E \Big| \underbrace{ \sum_{n \in [N]} \left[  \mu_\kk(\Q)(x_n)-\E_{x\sim \P} \mu_\kk(\Q)(x) \right] }_{=:S_N} \Big|^p \label{eq:markov_t1}.
\end{align}
$S_N$ is a sum of centered independent random variables and $T_1 - \E T_1 = \frac{1}{N} S_N$. By introducing the notation $Z_k =  \mu_\kk(\Q)(x_k)-\E_{x\sim \P} \mu_\kk(\Q)(x)$, $S_n = \sum_{k \in [n]} Z_k$ is a martingale w.r.t.\ the filtration $\mathscr{F}_n:=\sigma\left( (Z_k)_{k\in [n]}\right)$. Hence, one can apply the Burkholder's inequality (Theorem \ref{th:burkholder}) on $\{(S_n,\mathscr{F}_n)\}_{n\in[N]}$: it ensures the existence of a constant $C_p>0$ such that 
\begin{align}
     \E \left\lvert S_N \right\rvert^p &\leq C_p \E \left( \sum_{n \in [N]} Z_n^2 \right)^{p/2} 
     = C_p N^{\frac{p}{2}}\E \left(\frac{1}{N} \sum_{n \in [N]} Z_n^2 \right)^{p/2} 
     \stackrel{(a)}{\le}   C_p N^{\frac{p}{2}} \E \left(\frac{1}{N}\sum_{n\in [N]} |Z_n|^p\right)
     = C_p N^{\frac{p}{2}} \E |Z_N|^p\nonumber \\
     &= C_p N^{p/2}  \underbrace{ \E_{x'\sim \P} \lvert  \mu_\kk(\Q)(x')-\E_{x\sim \P} \mu_\kk(\Q)(x) \rvert^p }_{m'_p}, \label{eq:m_p_bound}
\end{align}
where in (a) we applied the Jensen inequality with the convex function $\phi(x):=x^{p/2}$. Let us show the finiteness of  $m'_p$.

\emph{Proof (finiteness of $m'_p$):}
\begin{itemize}[labelindent=0em,leftmargin=1em,topsep=0em,partopsep=0cm,parsep=0cm,itemsep=2mm]
    \item Let us first notice that assumption \eqref{eq:unbounded_as} (i.e., $\E_{x \sim \P} e^{ \lambda x^2} < \infty$ and $\E_{x \sim \Q} e^{ \lambda x^2} < \infty$ for all $\lambda \in \R$)
    implies that 
    \begin{align}
    \label{eq:unbounded_as_sup}
        \E_{x \sim \P, y \sim \Q} e^{ \lambda x y} < \infty  \quad \forall \lambda \in  {\Rpp}.
    \end{align}
    Indeed, taking $\lambda' \in {\Rpp}$ and using the inequality $xy \leq \frac{x^2+y^2}{2}$  for any $x, y \in \R$,  one gets 
    \begin{align}
    \E_{x \sim \P, y \sim \Q} e^{ \lambda' x y} &\leq \E_{x \sim \P, y \sim \Q} e^{ \frac{\lambda'}{2}\left( x^2+ y^2\right)}. \label{eq:means}     
    \end{align}
    By the independence of $x\sim\P$ and $y\sim\Q$, the r.h.s.\ of \eqref{eq:means} equals to  {$\E_{x \sim \P} e^{ \frac{\lambda'}{2}x^2 } \E_{y \sim \Q} e^{\frac{\lambda'}{2}y^2}$} which is finite by using  \eqref{eq:unbounded_as} with $\lambda = \frac{\lambda'}{2}$.
 \item 
        Let us show that $m'_p \leq 2^p \E_{x\sim \P,y \sim \Q} e^{ p b x y}$.  Indeed,
        \begin{align}
        m'_p  &=  \E_{x'\sim \P} \lvert  \mu_\kk(\Q)(x')-\E_{x\sim \P} \mu_\kk(\Q)(x) \rvert^p  
        \stackrel{(a)}{\le} 2^{p-1}  \E_{x'\sim \P}  \left(\lvert  \mu_\kk(\Q)(x') \rvert^p +   \lvert  \E_{x\sim \P} \mu_\kk(\Q)(x) \rvert^p   \right) \nonumber\\
        &\stackrel{(b)}{\le}  2^p  \E_{x'\sim \P} \lvert  \mu_\kk(\Q)(x') \rvert^p
        \stackrel{(c)}{=}  2^p\E_{x'\sim \P} \lvert \E_{ x \sim \Q}  \kk(x,x') \rvert^p  
        \stackrel{(d)}{\le } 2^p  \E_{x'\sim \P}  \E_{ x \sim \Q} \lvert \kk(x,x') \rvert^p             \stackrel{(e)}{=} 2^p \E_{x'\sim \P,x \sim \Q} e^{ p b x x'},  \label{eq:m'_p-bound}
        \end{align}
        where (a) follows from the convexity inequality $\lvert a+b \rvert^p \leq 2^{p-1}(\lvert a\rvert^p+\lvert b\rvert^p)$ for any $a,b \in \R, \ p\geq 1$ and the linearity of the integral, in (b) we applied the Jensen inequality with the convex function $\phi(x):=x^{p}$, in (c) the  definition of $\mu_\kk(\Q)(x)$ was used, (d) follows from the Jensen inequality, (e) is implied by the fact that $\kk(x,x')  = e^{b x x'}$. Applying  \eqref{eq:unbounded_as_sup} with $\lambda = b p>0$ (as $b>0$ and $p\ge 2$) implies that $\E_{x'\sim \P,x \sim \Q} e^{ p b x x'}<\infty$ which guarantees the finiteness of $m'_p$ by \eqref{eq:m'_p-bound}.
\end{itemize}

Substituting the bound \eqref{eq:m_p_bound} to \eqref{eq:markov_t1}, one gets that for any $t>0$
\begin{align}
\P(T_1 - \E T_1 >t) & \le \frac{1}{(Nt)^p} \E \left\lvert S_N \right\rvert^p \leq C_p m'_p \frac{1}{(N t^2)^{p/2}}. \label{eq:bound_T1}
\end{align}
Note: the same bound holds for the deviation below with $-t$
\begin{align}
\P(T_1 - \E T_1 <-t) & \le C_p m'_p \frac{1}{(N t^2)^{p/2}} \label{eq:bound_T1b}
\end{align}
by changing $T_1$ to $-T_1$ in the reasoning above.
\end{itemize}

Using in \eqref{eq:decomp_u} the bound \eqref{eq:bound_T1b} for $T_1$ with $t =\frac{\epsilon}{4}$ and the bound \eqref{eq:bound_T2} for $T_2$ with $t=\frac{\epsilon}{2}$ one gets
\begin{eqnarray*}
       \lefteqn{\P \left(\MMDshatEU{\P,\Q} -\MMDs{\P,\Q}  >\varepsilon \right) \le 
             \P\left( T_2-\E{T_2} > \frac{\varepsilon}{2} \right)+ \P \left(T_1-\E{T_1} <- \frac{\varepsilon}{4} \right)}\\
             && \le C_p m_p \underbrace{\left(\frac{2^2 \times 2}{N\varepsilon^2}\right)^{\frac{p}{2}}}_{\frac{(2\sqrt{2})^p}{\varepsilon^p N^{p/2}}} + C_p m'_p \frac{4^p}{\varepsilon^pN^{p/2}}    
            \le \frac{1}{\varepsilon^p N^{p/2}} \underbrace{\left[C_p m_p (2\sqrt{2})^p + C_p m'_p4^p\right]}_{=:C=:C_{p,\P,k}}.
\end{eqnarray*}
The lower deviation bound with $-\varepsilon$ follows by using the bound \eqref{eq:bound_T2b} for $T_2$ with $t=\frac{\varepsilon}{2}$ and the bound \eqref{eq:bound_T1} for $T_1$ with $t=\frac{\varepsilon}{4}$ 
\begin{align*}
       \P\left( \MMDshatEU{\P,\Q} -\MMDs{\P,\Q}  <-\varepsilon \right) & \le \P \left( T_2-\E{T_2} <- \frac{\varepsilon}{2} \right) + \P \left(T_1-\E{T_1} > \frac{\varepsilon}{4} \right) \le \frac{C_{p,\P,k}}{\varepsilon^p N^{p/2}}.
\end{align*}
\end{proof}

\subsection{Proof of Convergence of the V-statistic based Semi-Explicit MMD Estimator}\label{ssec:cv_VMMD}

\begin{proof}{(Convergence of $\MMDhatEV{\P,\Q}$)}

To understand the convergence behavior of $\MMDhatEV{\P,\Q}$,  let us start by considering the convergence of $\MMDshatEV{\P,\Q}$.

\paragraph{Convergence of $\MMDshatEV{\P,\Q}$:} 
Let us rewrite $\MMDshatEV{\P,\Q} -   \MMDs{\P,\Q}$ in terms of $\MMDshatEU{\P,\Q} -   \MMDs{\P,\Q}$.  
By the definition of  $\MMDshatEV{\P,\Q}$ and $\MMDshatEU{\P_N,\Q}$ [see \eqref{eq:MMD-hat-gen0:exp}-\eqref{eq:MMD-hat-gen:exp}], the two  estimators only differ in their first terms which we denote as
\begin{align*}
    T_1^V&:=\frac{1}{N^2}\sum_{{i,j\in [N]}}\kk\left(x_i,x_j\right), & 
    T_1^U &:=\frac{1}{N(N-1)}\sum_{\substack{i,j\in [N]\\ i\ne j}}\kk\left(x_i,x_j\right).
\end{align*}
These two terms are closely related; let us write $T_1^V$ in terms of $T_1^U$
\begin{align*}
T_1^V& = \frac{1}{N^2} \left( \sum_{\substack{i,j\in [N]\\ i\ne j}} \kk\left(x_i,x_j\right) + \sum_{i \in [N]}  \kk\left(x_i,x_i\right) \right)   = \underbrace{  \frac{N(N-1)}{N^2} }_{1-\frac{1}{N}} \underbrace{ \left( \frac{1}{N(N-1)} \sum_{\substack{i,j\in [N]\\ i\ne j}} \kk\left(x_i,x_j\right) \right)  }_{T_1^U} +\frac{1}{N^2} \sum_{i \in [N]}  \kk\left(x_i,x_i\right)
\end{align*}
which means that $T_1^V = \left(1-\frac{1}{N} \right) T_1^U + \frac{1}{N^2} \sum_{i \in [N]}  \kk\left(x_i,x_i\right)$. Denoting the second and third  common terms  of $\MMDshatEV{\P,\Q} $ and $\MMDshatEU{\P,\Q}$ by $T_2 := \E_{y \sim \Q}  \mu_\kk(\Q)(y)$ and $T_3: = -  2 \frac{\sum_{\substack{i\in [N]}} \mu_\kk(\Q)(x_i)}{N}$, we hence have
\begin{align*}
   \lefteqn{\MMDshatEV{\P,\Q} -   \MMDs{\P,\Q}
    = T_1^V  + T_2 +T_3  -   \MMDs{\P,\Q}}\\
    &= \left(1-\frac{1}{N} \right) T_1^U + T_2 +T_3 -   \MMDs{\P,\Q} + \frac{1}{N^2} \sum_{i \in [N]}  \kk\left(x_i,x_i\right)    \\
    &=  \left(1-\frac{1}{N} \right) \underbrace{ \left[  T_1^U + T_2 +T_3 -   \MMDs{\P,\Q} \right]}_{ \MMDshatEU{\P,\Q} -   \MMDs{\P,\Q} } + \frac{T_2}{N} + \frac{T_3}{N} - \frac{ \MMDs{\P,\Q}}{N}+ \frac{1}{N^2} \sum_{i \in [N]}  \kk\left(x_i,x_i\right).
\end{align*}
This implies that 
\begin{align}
     \MMDshatEV{\P,\Q} -   \MMDs{\P,\Q} &= \left(1-\frac{1}{N} \right)  \left[ \MMDshatEU{\P,\Q} -   \MMDs{\P,\Q}\right] + o_{a.s}\left(\frac{1}{\sqrt{N}}\right) \label{eq:MMDUMMDV}
\end{align}
as $T_2$ and $T_3$ are constants, and $ \frac{1}{N} \sum_{i \in [N]}  \kk\left(x_i,x_i\right) $ converge to a constant by the law of large numbers.  

Since $\MMDshatEU{\P,\Q} -   \MMDs{\P,\Q} = \mathcal{O}_{a.s} \left(\frac{1}{\sqrt{N}}\right)$,  
\eqref{eq:MMDUMMDV} means that $\MMDshatEV{\P,\Q}-  \MMDs{\P,\Q} = \mathcal{O}_{a.s.} \left(\frac{1}{\sqrt{N}}\right)$ also holds.

\paragraph{Convergence of $\MMDhatEV{\P,\Q}$} 
    Throughout the proof, we assume that $\MMD{\P,\Q}>0$. 
    In this case, 
\begin{align*}
\mathcal{O}_{a.s} \left(\frac{1}{\sqrt{N}}\right) &\stackrel{(*)}{=} \left| \MMDshatEV{\P,\Q} -\MMDs{\P,\Q} \right|\\ 
       & =\left| \MMDhatEV{\P,\Q} -\MMD{\P,\Q}\right| \Big[ \underbrace{\MMDhatEV{\P,\Q}}_{\ge 0 }+\underbrace{\MMD{\P,\Q} }_{> 0 }\Big] \\
       & \stackrel{}{\ge}   \left\lvert \MMDhatEV{\P,\Q} -\MMD{\P,\Q}\right\lvert \MMD{\P,\Q}  
\end{align*}
one gets that $\left| \MMDhatEV{\P,\Q} -\MMD{\P,\Q} \right| =\mathcal{O}_{a.s}\left(\frac{1}{\sqrt{N}}\right)$ by using in $(*)$ the previously established convergence $\MMDshatEV{\P,\Q}-  \MMDs{\P,\Q} = \mathcal{O}_{a.s.} \left(\frac{1}{\sqrt{N}}\right)$.


\end{proof}

\subsection{Proof of the Minimax Rate for the Unbounded Exponential Kernel}\label{ssec:MMD_minimax}

\begin{proof}{(Theorem~\ref{th:minimax_unbounded})}
Let $D = (x_n)_{n\in [N]} \iid \P \in \mathcal{P}$ and let $\widehat{\text{MMD}}_N$ denote any  estimator of $\MMD{\P,\Q}$ based on $D$. We are interested in the worst-case error (among all $\P,\Q \in \mathcal{P}$) of the best estimator $\widehat{\text{MMD}}_N$, in other words our target quantity is
\begin{align*}
\inf_{ \widehat{\text{MMD}}_N } \sup_{ \P,\Q \in \mathcal{P}} \P^N \left( \lvert \widehat{\text{MMD}}_N -  \MMD{\P,\Q} \lvert \geq s   \right),   \quad s >0,
\end{align*}
where $\P^N$ denotes the $N$-times product measure of $\P$.
Particularly, our goal is to show that $s = \frac{c}{\sqrt{N}}$ is a possible (hence optimal) rate, with some finite constant $c>0$.
Let us define a parameteric class of distributions $\mP_{\Theta}$, domain $\X$, and functional $F$
\begin{align}
     \mP_{\Theta}  &:=  \left\{[\mathcal{N}\left(m, \sigma^2\right)]^N\,:\, (m, \sigma) \in \Theta \right\}, & \Theta &:= \left\{ \left(m, \sigma\right) \in \R \times \left(0, \frac{1}{\sqrt{b}} \right)  \right\}, &
      \X &= \R^N, 
     & F(\theta)& = \MMD{ \P_\theta, \Q}
\end{align}
which we will use to invoke Theorem~\ref{th:minimax_general}. Here $\P_\theta := \mathcal{N}\left(m_\P, \sigma^2\right)$ where $\theta:=(m_\P, \sigma)\in \Theta$ and a fixed $\Q = \mathcal{N}\left(m_\Q, \sigma^2\right)$ is taken with $\theta_\Q = ( m_\Q, \sigma) \in \Theta$.\footnote{Notice that the variance parameter of $\P_\theta$ and $\Q_{\theta_\Q}$ are chosen to be identical.} First, let us notice that $\mP_{\Theta} \subset \mP$ since $\E_{x \sim \P} \sqrt{\kk(x,x)} < \infty$  means that $\sigma^2 < \frac{1}{b}$\footnote{A standard calculation shows that if $\P = \cN(m,\sigma)$, $\E_{x \sim \P} \sqrt{\kk(x,x)}  = e^{\frac{m^2}{1-b\sigma^2}} \frac{1}{ \sqrt{1-\sigma^2 b}}$ which is finite iff $\sigma^2< 1/b$.}. Using this inclusion one gets the following lower bound (which translated to a lower bound on the target quantity by taking the infimum over $\widehat{\text{MMD}}_N$)
\begin{align}
  \sup_{ \P,\Q \in \mathcal{P}} \P^N \left( \rvert \widehat{\text{MMD}}_N -  \MMD{\P,\Q} \lvert \geq s   \right) &\geq  \sup_{ \theta,\theta_\Q \in\Theta} \P_\theta^N \left(\lvert \widehat{\text{MMD}}_N -  \MMD{\P_{\theta},\Q_{\theta_\Q}} \lvert \geq s   \right) \nonumber \\
  &\geq  \sup_{ \theta \in \Theta} \P_\theta^N \left(\lvert \widehat{\text{MMD}}_N -  \MMD{\P_{\theta},\Q_{\theta_\Q}} \lvert \geq s   \right), \quad \forall
 \theta_\Q \in \Theta, \label{eq:inf_proba_bound}
\end{align}
which means that for any fixed $\theta_\Q \in \Theta$ we are in the realm of Theorem \ref{th:minimax_general}. To
apply the theorem, one needs (i) an upper bound on $D_{\text{KL}}\left(\Q_{\theta_\Q}^{\otimes N},\P_{\theta}^{ \otimes N}\right)$, and (ii) a lower bound on $\lvert F(\theta)-F(\theta_\Q)\rvert$. This is what we compute in the following. 
 
\begin{itemize}[labelindent=0em,leftmargin=1em,topsep=0.1em,partopsep=0cm,parsep=0cm,itemsep=2mm]
     \item \tb{Upper bound on }$D_{\text{KL}}\left(\Q_{\theta_\Q}^{\otimes N},\P_{\theta}^{ \otimes N}\right)$:
      Let $p$ and $q$ denote the pdf of $\P_{\theta}$ and $\Q_{\theta_\Q}$. Then the Kullback-Leibler divergence can be computed as 
       \begin{align}
            \lefteqn{D_{\text{KL}}\left( \Q_{\theta_\Q}^{\otimes N},\P_\theta^{ \otimes N}  \right)
             = \int_{ \R^N} \log\left(  \frac{ \prod_{n\in[N] } q(x_n) }{ \prod_{n\in[N] } p(x_n) } \right)   \prod_{j\in[N] } q(x_j) \  \d x_1 \dots \d x_N} \nonumber\\
            & = \sum_{n \in [N]}  \underbrace{\int_{\R^N}  \log\left(  \frac{q(x_n) }{ p(x_n) } \right)  \prod_{j\in[N] } q(x_j) \  \d x_1 \dots \d x_N}_{\underbrace{\int_{\R}    \log\left(  \frac{q(x_n) }{ p(x_n) } \right)  \d q(x_n)}_{D_{\text{KL}}\left(\Q_{\theta_\Q},\P_{\theta} \right)}\prod_{j\in[ N], j \ne n }  \underbrace{\int_{ \R}  q(x_j)   \d x_j}_{=1}}
            =  \sum_{n \in [N]}  D_{\text{KL}}\left(\Q_{\theta_\Q},\P_{\theta} \right) \stackrel{(a)}{=}  N \frac{(m_\P-m_\Q)^2}{2 \sigma^2} = \underbrace{\frac{a^2}{2 \sigma^2}}_{=:\alpha}, \label{eq:bound:KL}
    \end{align}
    where in (a) we used Lemma~\ref{lemma:explicit_KL}, and in (b) we assumed that 
    \begin{align}
         m_\P &= m_\Q + \frac{a}{\sqrt{N}} \label{eq:mP-mQ}
    \end{align}
    for some $a >0$.
     \item \tb{Lower bound on $\lvert F(\theta)-F(\theta_{\Q})\rvert$}: Since $F(\theta_\Q) = \MMD{ \Q, \Q} = 0$, it is sufficient  to compute $F(\theta) = \MMD{\P_{\theta},\Q_{\theta_\Q}}$. By  Lemma 1,
     one has 
     \begin{align}
            [F(\theta)]^2 =   \frac{1 }{\sqrt{1-\cbs^2}}   \left[ e^{ \frac{b m_\P^2}{1-\cbs} } +  e^{ \frac{b m_\Q^2}{1-\cbs} } - 2 e^{\frac{  b c \left( m_\P^2 + m_\Q^2 \right) +2 b m_\P m_\Q  }{ 2  ( 1-\cbs^2)  } } \right]
            \label{eq:F_theta}
     \end{align}
     where $\cbs=b\sigma^2$. We are going to show that
     \begin{align}
            \left[\MMD{\P_{\theta},\Q_{\theta_\Q}}\right] ^2  \ge\frac{(2 K)^2}{N} \label{eq:sMMD}
     \end{align} 
        for some constant $K>0$. Let $x =  \frac{a}{\sqrt{N}}$, in other words $m_\P = m_\Q + x$ (in accordance with \eqref{eq:mP-mQ}); we are going to rewrite the squared MMD in \eqref{eq:sMMD} as a function of $x=m_\P - m_\Q$. To do so we will apply a  Taylor expansion of the squared MMD around $x=0$. By introducing the notation 
        \begin{align*}
            f_1(x) &:= e^{ \frac{b (m_\Q+x)^2}{1-\cbs} } = e^{ \frac{b m_\P^2}{1-\cbs} }, &  f_2(x) &:= e^{  \frac{  2 b (1+ c)  x m_\Q + b c x^2  }{ 2  \left( 1-\cbs^2\right)  } },
        \end{align*}
        our target quantity writes as 
        \begin{align}\label{eq:MMD_f}
            \left[\MMD{\P_{\theta},\Q_{\theta_\Q}}\right]^2 &=    \frac{1}{\sqrt{1-\cbs^2}}   \left[ f_1\left( \frac{a}{\sqrt{N}} \right)+  f_1(0) - 2
            f_1(0) f_2\left( \frac{a}{\sqrt{N}}\right) \right].
        \end{align}
        Indeed, by substituting $m_\P = m_\Q +x$ in the third term of \eqref{eq:F_theta}, one gets
        \begin{align*}
            e^{\frac{  b c \left( m_\P^2 + m_\Q^2 \right) +2 b m_\P m_\Q  }{ 2  ( 1-\cbs^2)  } } &= 
            e^{\frac{  b c \left[ \left( m_\Q +x \right)^2 + m_\Q^2 \right] +2 b \left( m_\Q +x \right)m_\Q  }{ 2  ( 1-\cbs^2)  } }   =\underbrace{e^{\frac{  2 b c m_\Q^2 +2 b m_\Q^2 }{ 2  ( 1-\cbs^2)  } }}_{e^{\frac{2b(c+1)m_{\Q}^2}{2(1-\cbs)(1+c)}}} e^{  \frac{  b c \left( 2 x m_\Q + x^2 \right) +2 b x m_\Q  }{ 2  ( 1-\cbs^2)  } } = \underbrace{ e^{\frac{   b  m_\Q^2  }{  1-\cbs  } } }_{f_1(0)} \underbrace{ e^{  \frac{  2 b (1+ c)  x m_\Q + b c x^2  }{ 2  \left( 1-\cbs^2\right)  } } }_{f_2(x)} .
        \end{align*}
       Let us first form the second-order Taylor expansion of $f_1$ and $f_2$ around $x=0$; for this approximation the derivatives are 
       \begin{align*}
            f'_1(x) & = f_1(x) \left[ \frac{2b}{1-\cbs} (m_\Q+x) \right], & f''_1(x)  &= f_1(x) \left( \frac{2b}{1-\cbs}  + \left[ \frac{2b}{1-\cbs} (m_\Q+x) \right]^2 \right), \\
            f'_2(x) &= f_2(x) \left( \frac{b m_\Q}{1-\cbs} + \frac{b c x}{1-\cbs^2}  \right),  & 
            f''_2(x) &= f_2(x) \left[ \frac{b c }{1-\cbs^2} +  \left( \frac{b m_\Q}{1-\cbs} + \frac{b c x}{1-\cbs^2}  \right)^2 \right],
        \end{align*}
        which means that for $x=0$ one has
        \begin{align*}
        f_1'(0)& =f_1(0) \frac{2b m_\Q }{1-\cbs} , & f''_1(0)  &= f_1(0) \left[ \frac{2b}{1-\cbs}  + \left( \frac{2b}{1-\cbs} m_\Q \right)^2 \right], \\
        f'_2(0) &= \underbrace{f_2(0)}_{=1} \frac{b m_\Q}{1-\cbs} , & 
        f''_2(0) &=  \underbrace{f_2(0)}_{=1} \left[ \frac{b c }{1-\cbs^2} +  \left( \frac{b m_\Q}{1-\cbs}  \right)^2 \right].
        \end{align*}
    Consequently, the 2nd-order Taylor expansion of $f_1$ and $f_2$ takes the form 
    \begin{align*}
            f_1\left( \frac{a}{\sqrt{N}} \right) &=   f_1(0)  + f_1(0) \frac{2b m_\Q }{1-\cbs}  \frac{a}{\sqrt{N}} +  f_1(0) \left[ \frac{2b}{1-\cbs}  + \left( \frac{2b}{1-\cbs} m_\Q \right)^2 \right] \frac{a^2}{2 {N}} + \smallO \left(  \frac{a^2}{2 {N}}  \right) ,\\
            f_2\left( \frac{a}{\sqrt{N}} \right)& =  1 + \frac{b m_\Q }{1-\cbs}  \frac{a}{\sqrt{N}} +  \left[ \frac{b c }{1-\cbs^2} +  \left( \frac{b m_\Q}{1-\cbs}  \right)^2 \right] \frac{a^2}{2 {N}} + \smallO \left(  \frac{a^2}{2 {N}}  \right).
    \end{align*}
        Using these expansions in \eqref{eq:MMD_f}, the $f_1(0) $ and $ f_1(0) \frac{2b m_\Q }{1-\cbs}  \frac{a}{\sqrt{N}} $ terms simplify and one gets
        \begin{align*}
             \left[\MMD{\P_{\theta},\Q_{\theta_\Q}}\right]^2 
             &= \frac{1 }{\sqrt{1-\cbs^2}} f_1(0)\Bigg(  \underbrace{\left[\frac{2b}{1-\cbs}+\left(\frac{2bm_\Q}{1-\cbs}\right)^2\right] -2 \left[\frac{\cbs b}{1-\cbs^2} + \left(\frac{bm_\Q}{1-\cbs}\right)^2\right]}_{=  \underbrace{\frac{2b}{1-\cbs}}_{=\frac{2b(1+c)}{1-\cbs^2}} - \frac{2\cbs b}{1-\cbs^2} + 2\left(\frac{bm_\Q}{1-\cbs}\right)^2} \Bigg)  \frac{a^2}{2 {N}} + \smallO \left(  \frac{a^2}{2 {N}}  \right)\\
             & = \frac{1 }{\sqrt{1-\cbs^2}} f_1(0) \left[\frac{2b}{1-\cbs^2} + 2\left(\frac{bm_\Q}{1-\cbs}\right)^2 \right]\frac{a^2}{2 {N}} + \smallO \left(  \frac{a^2}{2 {N}}  \right).
        \end{align*}
        This means that the term in $\frac{a^2}{2 {N}} $ will be smaller than the remaining term $ \smallO \left(  \frac{a^2}{2 {N}} \right)$ for  large enough $N$. Hence there exists a constant $K>0$ such that
        \begin{align*}
            \left[\MMD{\P_{\theta},\Q_{\theta_\Q}}\right] ^2  \ge \frac{(2K)^2 }{ {N}}. 
        \end{align*} 
        Hence we have that
        \begin{align}
        \lvert F(\theta) - F(\theta_\Q) \rvert  = \lvert F(\theta)  \rvert \geq \frac{2K}{\sqrt{N}} :=2 s. \label{eq:bound:F}
        \end{align}
\end{itemize}    
By using the derived bounds \eqref{eq:bound:KL} and \eqref{eq:bound:F},  Theorem~\ref{th:minimax_general} can be applied with $\alpha = \frac{a^2}{2 \sigma^2}$ and $s= \frac{K}{\sqrt{N}}$, and the bound \eqref{eq:inf_proba_bound} implies that
\begin{align*}
\inf_{ \widehat{\text{MMD}}_N } \sup_{ \P,\Q \in \mathcal{P}} \P \left( \rvert \widehat{\text{MMD}}_N -  \MMD{\P,\Q} \lvert \geq \frac{K}{\sqrt{N}}   \right)& \geq \max\left( \frac{e^{- \frac{a^2}{2 \sigma^2} }}{4},   \frac{1-\sqrt{  \frac{a^2}{2 \sigma^2} }}{2}  \right).
\end{align*}
Since the bound is valid for any $\sigma$ for which $\sigma^2  <\frac{1}{b}$, by continuity one can also take the limit $\sigma^2  =\frac{1}{b}$ for which the lower bound is maximized and writes as  
\begin{align*}
 \max\left( \frac{e^{- \frac{a^2 b}{2} }}{4},   \frac{1-\sqrt{  \frac{a^2 b}{2}}}{2}  \right) .
\end{align*}
\end{proof}

\section{ANALYTICAL MEAN EMBEDDINGS}\label{sec:two}
In Section~\ref{ssec:analytical_meanemb_supp} we state our novel analytical mean-embedding results, followed their proofs (Section~\ref{ssec:mean_emb_proof}). Auxiliary results related to Lemma~\ref{lemma:mean_emb_beta_matern} are given in Section~\ref{ssec:mean_emb_technical}.

\subsection{Our Results on Analytical Mean Embeddings}\label{ssec:analytical_meanemb_supp} 

Below we present our results on analytical mean-embeddings obtained for  the (Gaussian-exponentiated, Gaussian) and (Mat{\'e}rn, beta) kernel-distribution pairs.
\begin{lemma}[Mean embedding:  Gaussian-exponentiated  kernel - Gaussian target]\label{lemma:mean_emb_gaussian_gaussian_exp}
Let the target distribution be Gaussian $q(x)=\frac{1}{ \sqrt{2 \pi \sigma^2}} \mathrm {e}^{-{\frac {\left(x-m\right)^2}{2 \sigma^2 }}}$,  
the kernel be Gaussian-exponentiated  $\kk(x,y) = e^{-a(x-y)^2 + b xy}$ where  $m\in \R$, $\sigma \in \Rpp$ $ a\ge 0, b\ge 0$. Then the mean embedding 
$\mu_\kk(\Q)$ can be computed analytically as 
\begin{align}
\mu_\kk(\Q)(x) &=\frac{e^{ - \frac{a (x-m)^2}{ 1+2 a \sigma^2} + \frac{2 b m x + b(b+4a)\sigma^2 x^2}{2(1+2 a \sigma^2)}}}{\sqrt{1+2a \sigma^2}}. \label{eq:mean-embedding:Gaussian-target,Gaussian-exponentiated-kernel}
\end{align}
\end{lemma}

\begin{lemma}[Mean embedding: Mat{\'e}rn kernel - beta target]\label{lemma:mean_emb_beta_matern}
Let the target distribution be beta $q(x) ={\frac{1}{B (\alpha ,\beta )}} x^{\alpha -1}(1-x)^{\beta -1} \I_{[0,1]}(x) $ with $\alpha\in \Rpp$, $\beta\in \Rpp$, and let 
the kernel be Mat{\'e}rn with half-integer $\nu$ ($\nu = p + \frac{1}{2}$, $p\in \N$), $\sigma_0 \in \Rpp$, $\sigma \in \Rpp$  
\begin{align*}
\kk(x,y) = \sigma_0^2 e^{-{\frac {{\sqrt {2p+1}}|x-y|}{\sigma }} }{\frac {p!}{(2p)!}}\sum _{i=0}^{p}{\frac {(p+i)!}{i!(p-i)!}}\left({\frac {2{\sqrt {2p+1}}|x-y|}{\sigma }}\right)^{p-i}.
\end{align*}
Then the mean embedding $\mu_\kk(\Q)$ can be analytically computed as
\begin{align*}
\mu_\kk(\Q)(x)& = \frac{\sigma_0^2}{B(\alpha,\beta) } \frac{p!}{(2p)!} \sum _{i=0}^{p}{\frac {(p+i)!}{i!(p-i)!}} \left({\frac {2{\sqrt {2p+1}}}{\sigma }}\right)^{p-i} \times \\
& \quad \sum_{k=0}^{p-i}  {p-i \choose k}  x^k  \bigg[   (-1)^{p-i-k} e^{-{\frac {{\sqrt {2p+1}}x}{\sigma }}} E_1^{\frac {{\sqrt {2p+1}}}{\sigma }}\left( (0\vee x) \wedge 1,p-i-k+\alpha-1,\beta-1 \right)  \\
&\hspace*{2.5cm} + (-1)^k e^{{\frac {{\sqrt {2p+1}}x}{\sigma }}} E_2^{\frac {{\sqrt {2p+1}}}{\sigma }}\left( (0\vee x) \wedge 1,p-i-k+\alpha-1,\beta-1 \right) \bigg],
\end{align*}
where for $a,b > - 1$, 
$E_1^{\lambda}$ and $E_2^{\lambda}$ are defined as
\begin{align*}
E^\lambda_1(z,a,b) & = \int_0^{z}  y^a (1-y)^b e^{\lambda y}\d y,  \quad E^\lambda_2(z,a,b) =\int_{z}^1  y^a (1-y)^b e^{-\lambda y}\d y
\end{align*}
and 
 can be evaluated using Lemma~\ref{lemma:E_seq_sum}.
\end{lemma}

\subsection{Proofs of Analytical Mean Embeddings}\label{ssec:mean_emb_proof}

Below we provide the proofs of our analytical mean-embedding results.
\begin{proof}{(Lemma~\ref{lemma:mean_emb_gaussian_gaussian_exp}; mean embedding:  Gaussian target - Gaussian-exponentiated kernel)}
Our target quantity is 

\begin{align*}
\mu_\kk(\Q)(x) &= \int_\R \frac{1}{ \sqrt{2 \pi \sigma^2}}  {e}^{-{\frac {\left(y-m\right)^2}{2 \sigma^2 }}} e^{-a(x-y)^2 + b xy} \d y.
\end{align*}
Completing the square, one gets
\begin{eqnarray*}
\lefteqn{\frac {\left(y-m\right)^2}{2 \sigma^2 }+a(x-y)^2 - b xy  = \frac {\left(y-m\right)^2 + 2 \sigma^2 a(x-y)^2 - 2 \sigma^2b xy}{2 \sigma^2 }} \\
&&=\frac{\left(1+2a\sigma^2\right) y^2 - 2\left(m+2 a \sigma^2 x + b \sigma^2 x\right) y + 2 a \sigma^2 x^2 + m^2 }{2 \sigma^2}  \\
&& = \underbrace{ \frac{ (1+2a\sigma^2)}{2 \sigma^2} }_{=:\frac{1}{2 {\sigma^*}^2} } \bigg( y^2 -  2 \underbrace{\frac{m+2 a \sigma^2 x + b \sigma^2 x}{1+2a\sigma^2}   }_{=:m^*} y \bigg) + \frac{2 a \sigma^2 x^2 + m^2}{2 \sigma^2} \\
&& = \frac{1}{2 {\sigma^*}^2} \left[ \left( y - m^* \right)^2 -{m^*}^2\right]+ \frac{2 a \sigma^2 x^2 + m^2}{2 \sigma^2}\\
&& =  \frac{1}{2 {\sigma^*}^2}  \left( y - m^* \right)^2 \underbrace{ - \frac{(m+2 a \sigma^2 x + b \sigma^2 x)^2}{2 \sigma^2(1+2a\sigma^2)} + \frac{2 a \sigma^2 x^2 + m^2}{2 \sigma^2} }_{=:-c^*}.
\end{eqnarray*}

Bringing the two terms in $c^*$ to common denominator, after simplification on arrives at 
\begin{align*}
c^* &
= \frac{\left(m+2 a \sigma^2 x + b \sigma^2 x\right)^2 - \left(2 a \sigma^2 x^2 + m^2\right)\left(1+2a \sigma^2\right)}{2 \sigma^2\left(1+2a\sigma^2\right)} = \frac{4 a \sigma^2 m x - 2 a \sigma^2(m^2 + x^2) +  b \sigma^2 x \left( b \sigma^2 x + 4 a \sigma^2 x  + 2 m \right) }{2 \sigma^2(1+2a\sigma^2)}\\
& =  \frac{   - 2 a \sigma^2  (m-x)^2 + b \sigma^2 x \left( b \sigma^2 x + 4 a \sigma^2 x  + 2 m \right) }{2 \sigma^2\left(1+2a\sigma^2\right)}  = - \frac{a (x-m)^2}{ 1+2 a \sigma^2}  + \frac{ 2 bm x + b \sigma^2(b + 4a) x^2}{ 2(1+2a\sigma^2) }.
\end{align*}

This means that our target quantity can be rewritten as
\begin{align*}
     \mu_\kk(\Q)(x) &= \int_\R \frac{1}{ \sqrt{2 \pi \sigma^2}} e^{-\frac{1}{2 {\sigma^*}^2 }  \left( y - m^* \right)^2 + c^*} \d y
     = \frac{\sigma^*}{\sigma} e^{c^*} \underbrace{\int_\R \frac{1}{ \sqrt{2 \pi (\sigma^*)^2}} e^{-\frac{1}{2 {\sigma^*}^2 }  \left( y - m^* \right)^2} \d y}_{=1} 
   = \frac{e^{ - \frac{a (x-m)^2}{ 1+2 a \sigma^2} + \frac{2 b m x + b(b+4a)\sigma^2 x^2}{2(1+2 a \sigma^2)}}}{\sqrt{1+2a \sigma^2}}.
\end{align*}
\end{proof}

\begin{proof}{(Lemma~\ref{lemma:mean_emb_beta_matern}; mean embedding: beta target - Mat{\'e}rn kernel)}
Our target quantity is 
\begin{align*}
& \quad \mu_\kk(\Q)(x) \\
&  = \frac{\sigma_0^2}{B(\alpha,\beta) }  \int_{0}^1 e^{-{\frac {{\sqrt {2p+1}}|x-y|}{\sigma }} } {\frac {p!}{(2p)!}}\sum _{i=0}^{p}{\frac {(p+i)!}{i!(p-i)!}} \left({\frac {2{\sqrt {2p+1}}|x-y|}{\sigma }}\right)^{p-i}   y^{\alpha-1} (1-y)^{\beta-1} \I_{y \in [0,1] }  \d y \\
&\stackrel{(*)}{=}  \frac{\sigma_0^2}{B(\alpha,\beta) } \frac {p!}{(2p)!} \sum _{i=0}^{p}{\frac {(p+i)!}{i!(p-i)!}} \left({\frac {2{\sqrt {2p+1}}}{\sigma }}\right)^{p-i}  \bigg[  \int_0^{(0\vee x) \wedge 1} (x-y)^{p-i}  y^{\alpha-1} (1-y)^{\beta-1} e^{-{\frac {{\sqrt {2p+1}}(x-y)}{\sigma }}} \d y 
\\
& \hspace*{7cm}
 + \int_{(0\vee x) \wedge 1}^1 (y-x)^{p-i}  y^{\alpha-1} (1-y)^{\beta-1} e^{-{\frac {{\sqrt {2p+1}}(y-x)}{\sigma }}} \d y \bigg].
\end{align*}

where in $(*)$ we applied the decomposition trick: 
for $f: \R \rightarrow \R$ and $x \in \R$,
\begin{align*}
\int_0^1 f(|x-y|)\d y =  \int_{0}^{(0 \vee x)\wedge 1}f(x-y)\d y + \int_{(0 \vee x)\wedge 1}^1 f(y-x) \d y.
\end{align*}

Using the binomial formula on $(x-y)^{p-i}= \sum_{k=0}^{p-i}  {p-i \choose k}   x^k (-y)^{p-i-k}$ and  on $(y-x)^{p-i}$, one can rewrite the two integrals as 
\begin{eqnarray*}
 \lefteqn{\int_0^{(0\vee x) \wedge 1} (x-y)^{p-i}  y^{\alpha-1} (1-y)^{\beta-1} e^{-{\frac {{\sqrt {2p+1}}(x-y)}{\sigma }}} \d y  =}\\ 
 &&=\sum_{k=0}^{p-i}  {p-i \choose k}(-1)^{p-i-k}  x^k  \int_0^{(0\vee x) \wedge 1} y^{p-i-k} y^{\alpha-1} (1-y)^{\beta-1} e^{-{\frac {{\sqrt {2p+1}}(x-y)}{\sigma }}} \d y \\
 &&=  \sum_{k=0}^{p-i}  {p-i \choose k} (-1)^{p-i-k}  x^k  e^{-{\frac {{\sqrt {2p+1}}x}{\sigma }}} \underbrace{ \int_0^{(0\vee x) \wedge 1} y^{p-i-k+\alpha-1} (1-y)^{\beta-1} e^{{\frac {{\sqrt {2p+1}}y}{\sigma }}} \d y }_{ 
  E_1^{\frac {{\sqrt {2p+1}}}{\sigma }}\left( (0\vee x) \wedge 1,p-i-k+\alpha-1,\beta-1 \right) },
\end{eqnarray*}
and 
\begin{eqnarray*}
 \lefteqn{\int_{(0\vee x) \wedge 1}^1 (y-x)^{p-i}  y^{\alpha-1} (1-y)^{\beta-1} e^{-{\frac {{\sqrt {2p+1}}(y-x)}{\sigma }}} \d y=}  \\
&&= \sum_{k=0}^{p-i}  {p-i \choose k} (- x)^k  \int_0^{(0\vee x) \wedge 1} y^{p-i-k} y^{\alpha-1} (1-y)^{\beta-1} e^{-{\frac {{\sqrt {2p+1}}(y-x)}{\sigma }}} \d y \\
 & &=\sum_{k=0}^{p-i}  {p-i \choose k} (-x)^k  e^{{\frac {{\sqrt {2p+1}}x}{\sigma }}} \underbrace{ \int_{(0\vee x) \wedge 1}^1 y^{p-i-k+\alpha-1} (1-y)^{\beta-1} e^{-{\frac {{\sqrt {2p+1}}y}{\sigma }}} \d y }_{ 
  E_2^{\frac {{\sqrt {2p+1}}}{\sigma }}\left( (0\vee x) \wedge 1,p-i-k+\alpha-1,\beta-1 \right) },
\end{eqnarray*}
Hence,
\begin{align*}
\mu_\kk(\Q)(x)& = \frac{\sigma_0^2}{B(\alpha,\beta) } \frac {p!}{(2p)!} \sum _{i=0}^{p}{\frac {(p+i)!}{i!(p-i)!}} \left({\frac {2{\sqrt {2p+1}}}{\sigma }}\right)^{p-i}\times\\
& \quad \sum_{k=0}^{p-i}  {p-i \choose k}  x^k  \bigg[   (-1)^{p-i-k} e^{-{\frac {{\sqrt {2p+1}}x}{\sigma }}} E_1\left( (0\vee x) \wedge 1,p-i-k+\alpha-1,\beta-1 \right)  \\
&\hspace*{2.5cm} + (-1)^k e^{{\frac {{\sqrt {2p+1}}x}{\sigma }}} E_2\left( (0\vee x) \wedge 1,p-i-k+\alpha-1,\beta-1 \right) \bigg].
\end{align*}
\end{proof}

\subsection{Auxiliary Results for Analytical Mean Embeddings}\label{ssec:mean_emb_technical}

In this section, we present an approximation formula on $E_1^\lambda$ and $E_2^\lambda$  which can be used to evaluate the (Mat{\'e}rn kernel, beta target) mean-embedding in
 Lemma~\ref{lemma:mean_emb_beta_matern}.

\begin{lemma}[Infinite sum formulation of $E_1^\lambda$ and $E_2^\lambda$, and truncation error control]\label{lemma:E_seq_sum}
For $z \in [0,1]$, $a>-1$, $b>-1$, let 
$E_1^\lambda(z,a,b)= \int_0^{z}  y^a (1-y)^b e^{\lambda y}\d y$ and $E_2^\lambda(z,a,b)=\int_{z}^1  y^a (1-y)^b e^{-\lambda y}\d y$. Then 
we have the following infinite sum formulation: 
            \begin{align*}
        E_1^\lambda(z,a,b) & =  \sum_{k=0}^\infty \frac{\lambda ^k}{k!} B_{\text{inc}}\left( a+k+1,b+1,z\right), \\
            E_2^\lambda(z,a,b) & = \sum_{k=0}^\infty \frac{(-\lambda)^k}{k!} \left[B\left( a+k+1, b+1\right) - B_{\text{inc}}\left( a+k+1, b+1,z\right) \right].
            \end{align*}
   Let $K \in \N$ be fixed and let us denote the truncated $E_1^\lambda$ and $E_2^\lambda$ by
\begin{align*}
E_{1}^{\lambda,tr}& = \sum_{k=0}^K \frac{\lambda^k}{k!}  B_{\text{inc} } \left( a+k+1,b+1,z\right),\\
 E_{2}^{\lambda,tr}&= \sum_{k=0}^K \frac{(-\lambda)^k}{k!} \left[{B\left( a+k+1, b+1 \right)} - { B_{\text{inc}}\left( a+k+1, b+1,z\right)} \right].
\end{align*}
Then the following bounds hold for the truncation errors
\begin{align*}
  E_{1}^\lambda - E_{1}^{\lambda,tr}& \leq  \frac{ \lambda^{K+1}e^{\lambda z} }{(K+1)!} B_{\text{inc}}\left( a+K+2,b+1,z\right) :=   \mathcal{E}^{\lambda,tr}_1 ,\\
     E_{2}^\lambda - E_{2}^{\lambda,tr}& \leq \frac{(-\lambda )^{K+1}e^{-\lambda z}}{(K+1)! }   \left[ B\left( a+K+2, b+1\right) - B_{\text{inc}}\left( a+K+2,b+1,z\right) \right] :=  \mathcal{E}^{\lambda,tr}_2.
\end{align*}

\end{lemma}

\begin{proof}{(Lemma~\ref{lemma:E_seq_sum})}
For $z \in [0,1]$, $a>-1$ and $b>-1$, let 
\begin{align*}
E_1^\lambda(z,a,b)&= \int_0^{z}  y^a (1-y)^b e^{\lambda y}\d y, \quad E_2^\lambda(z,a,b)=\int_{z}^1  y^a (1-y)^b e^{-\lambda y}\d y. 
\end{align*}

The infinite sum formulations follow from the exponential series expansions $e^{\lambda y} = \sum_{k=0}^\infty \frac{(\lambda y)^k}{k!}$ and $e^{-\lambda y} = \sum_{k=0}^\infty \frac{(-\lambda y)^k}{k!}$:
\begin{align*}
     E_1^\lambda(z,a,b) & = \int_0^z y^a (1-y)^{b } e^{\lambda y} \d y = \int_0^z y^a (1-y)^{b} \sum_{k=0}^\infty \frac{(\lambda y)^k}{k!} \d y
 = \sum_{k=0}^\infty \frac{\lambda^k}{k!} \underbrace{\int_0^z y^{a+k} (1-y)^{b } \d y }_{ B_{\text{inc}}\left( a+k+1, b+1,z\right)},\\
 E_{2}^\lambda(z,a,b) &= \int_z^1 y^a (1-y)^{b+1} e^{-\lambda y} \d y =\sum_{k=0}^\infty \frac{(-\lambda)^k}{k!} {\int_z^1 y^{a+k} (1-y)^{b } \d y }\\
 & = \sum_{k=0}^\infty \frac{(-\lambda)^k}{k!} \Bigg[\underbrace{\int_0^1  y^{a+k} (1-y)^{b} \d y }_{B\left( a+k+1,b+1\right)} - \underbrace{ \int_0^z  y^{a+k} (1-y)^{b } \d y }_{ B_{\text{inc}}\left( a+k+1, b+1,z\right)} \Bigg].
\end{align*}

Let us now fix $K \in \N$, and truncate $E_{1}^\lambda(z,a,b)$ and $E_{2}^\lambda(z,a,b)$  to the first $K+1$ terms: 
\begin{align*}
     E_{1}^{\lambda,tr} &= \sum_{k=0}^K \frac{\lambda^k}{k!}  B_{\text{inc} } \left( a+k+1,b+1,z\right), \\
     E_{2}^{\lambda,tr} &= \sum_{k=0}^K \frac{(-\lambda)^k}{k!} \left[{B\left( a+k+1, b+1 \right)} - { B_{\text{inc}}\left( a+k+1, b+1,z\right)} \right].
\end{align*}

By the Taylor-Lagrange theorem, in case of
\begin{itemize}
     \item  $E_{1}^{\lambda}$: for any $y \in [0,z]$ there is a $y_K \in (0,y)$ such that
        \begin{align*}
            e^{\lambda y} & = \sum_{k=0}^K \frac{(\lambda y)^k}{K!} + e^{\lambda y_K} \frac{(\lambda y)^{K+1}}{(K+1)!}.
        \end{align*}
     \item $E_{2}^{\lambda}$: for any $y \in [z,1]$ there is a $y_K' \in (0,y)$ such that
        \begin{align*}
                e^{-\lambda y} & = \sum_{k=0}^K \frac{(-\lambda y)^k}{K!} + e^{-\lambda y'_K} \frac{(-\lambda y)^{K+1}}{(K+1)!}.
        \end{align*}
\end{itemize}

Hence, the truncation errors can be bounded as 
\begin{align*}
E_{1}^\lambda - E_{1}^{\lambda, tr}& = \int_0^z y^a(1-y)^{b}  e^{\lambda y_K} \frac{(\lambda y)^{K+1}}{(K+1)!} \d y \stackrel{(a)}{\le} \int_0^z y^a(1-y)^{-1/2}  e^{\lambda z} \frac{(\lambda y)^{K+1}}{(K+1)!} \d y \\
&= \frac{ \lambda^{K+1} e^{\lambda z} }{(K+1)!} \underbrace{\int_0^z y^{a+K+1} (1-y)^{b} \d y}_{B_{\text{inc}}\left( a+K+2,b+1,z\right)}, \\
E_{21}^\lambda - E_{21}^{\lambda,tr}& = \int_z^1 y^a(1-y)^{b}  e^{-\lambda y'_K} \frac{(-\lambda y)^{K+1}}{(K+1)!} \d y 
\stackrel{(b)}\le \int_z^1 y^a(1-y)^b  e^{-\lambda z}\frac{(-\lambda y)^{K+1}}{(K+1)!} \d y\\
&= \frac{(-\lambda )^{K+1}}{(K+1)!}  e^{-\lambda z}  \int_z^1 y^{a+K+1} (1-y)^{b} \d y\\
&
=  \frac{(-\lambda )^{K+1}}{(K+1)!}  e^{-\lambda z}  \left[\underbrace{\int_0^1 y^{a+K+1} (1-y)^{b} \d y}_{B\left( a+K+2, b+1 \right)} - \underbrace{\int_0^z y^{a+K+1} (1-y)^{b} \d y}_{B_{\text{inc}}\left( a+K+2, b+1,z\right)} \right]\\
&= \frac{(-\lambda )^{K+1}}{(K+1)!}  e^{-\lambda z} \left[ B\left( a+K+2, b+1 \right) - B_{\text{inc}}\left( a+K+2, b+1,z\right) \right],
\end{align*}
where in (a) we used that $e^{\lambda y_K} \le e^{\lambda z}$ and in (b) that $e^{-\lambda y'_K} \le e^{-\lambda z}$.
\end{proof}

\section{EXTERNAL STATEMENTS}  \label{sec:external-statements}

    \setcounter{theorem}{0}
    \renewcommand{\thetheorem}{\Alph{section}\arabic{theorem}}

    \setcounter{theorem}{0}
    \renewcommand{\thelemma}{\Alph{section}\arabic{lemma}}
This section contains external statements used in the proofs of our concentration results.
 
\begin{theorem}[Hoeffding inequality for U-statistic; \cite{hoeffding1963}, \cite{pitcan2017note}] \label{th:Hoeffding_thm}
Assume that we have $n$ i.i.d.\ samples $\{X_i\}_{i\in [n]} \sim \P$. 
Let $I_m^n$ be the set $m$-tuples chosen without repetition from $[n]$. 
Suppose that $h : \R^m \rightarrow \R$ is bounded: $a \leq h(x_{1}, \dots,x_{m}) \leq b$ for all $ (x_1,\dots, x_m)$.
We denote $m_h = \E h (X_1,\dots,X_m)$ and its U-statistic based estimator  
 $U_n  = \frac{1}{\binom{n}{m} } \sum_{(i_1, \dots,i_m) \in I_m^n} h(X_{i_1}, \dots,X_{i_m})$.
 Then, for any $\varepsilon>0$
\begin{align*}
\P\left(  U_n - m_h >  \varepsilon  \right) \leq e^{ -\frac{2 \left\lfloor \frac{n}{m}\right\rfloor \varepsilon^2}{(b-a)^2} },
\end{align*}
and the same deviation bound holds for $-\varepsilon$ below, i.e.\ $\P\left(  U_n - m_h <  -\varepsilon  \right) \leq e^{ -\frac{2 \left\lfloor \frac{n}{m}\right\rfloor \varepsilon^2}{(b-a)^2} }$.
\end{theorem}

\begin{lemma}[Generalized Markov's inequality; (2.1) in \cite{boucheron2013concentration}]
\label{lemma:gen_markov_ineq}
Let $\phi$ denote a nondecreasing and nonnegative function defined on $I \subseteq \R$ and let $Y$ denote a random variable taking values in $I$. Then Markov's inequality implies that for every $t\in I$ with $\phi(t)>0$
\begin{align*}
\P\left( Y \ge t \right) \leq \frac{\E\phi(Y)}{\phi(t)}.
\end{align*}
\end{lemma}

\begin{theorem}[Burkholder's inequality; Theorem~2.10 in \cite{hall2014martingale}]
\label{th:burkholder}
Assume that $\{( S_i, \cF_i)\}_{i \in[n]}$ is a martingale sequence and its filtration,  $1<p<\infty$. 
Let the associated martingale increments be denoted by $X_1=S_1$ and $X_i = S_i-S_{i-1}, 2 \leq i \leq n$. Then there exist a constant $C_p$ depending on $p$ such that 
\begin{equation*}
\E{\bigg\lvert S_n \bigg\rvert^p }\leq  C_p \E{ \bigg\lvert \sum_{i=1}^n X_i^2\bigg\rvert^{p/2} } .
\end{equation*}
\end{theorem}

\begin{theorem}[Theorem~2.2 in \cite{tsybakov08introduction}] \label{th:minimax_general}
Let $\X$ and  $\Theta$ denote two measurable spaces. Let $F: \Theta \rightarrow \R$ be a functional. Let $\mathcal{P}_\Theta = \{ \P_\theta\, : \, \theta \in \Theta \} $ be a class of probability measures on $\X$ indexed by $\Theta$. We observe the data $D$ 
distributed according  $\P_\theta \in \mathcal{P}_\Theta$  with some unknown $\theta$. The goal is to estimate $F(\theta)$. Let $\hat{F}:= \hat{F}(D)$ be an estimator of $F(\theta)$ based on $D$. 
Assume that there exist $\theta_0, \theta_1 \in \Theta$ such that $|F(\theta_0)-F(\theta_1)|\geq 2s >0$ and $D_{\text{KL}}\left(\P_{\theta_1},\P_{\theta_0} \right)\le \alpha$ with $0<\alpha <\infty$. Then
\begin{align*}
\inf_{\hat{F}} \sup_{\theta \in \Theta} \P_\theta \left( |\hat{F} - F(\theta) | \geq s \right) \geq \max\left(  \frac{e^{-\alpha}}{4},  \frac{1-\sqrt{\alpha/2}}{2} \right).
\end{align*}
Note: Typically $\X$, $D$ and $\P_{\theta}$ depend on the sample size $N$.
\end{theorem}

\begin{lemma}[Kullback-Leibler divergence for univariate Gaussian variables; page~13 in \cite{duchi07derivations}] \label{lemma:explicit_KL} 
Let $\P = \mathcal{N}(m_\P,\sigma_\P)$, $\Q = \mathcal{N}(m_\Q,\sigma_\Q)$,  $m_\P, m_\Q \in \R$, $\sigma_\P, \sigma_\Q \in \R^{>0}$. Then
\begin{align*}
D_{\text{KL}} \left(\P, \Q \right)   &= \log \left(\frac{\sigma_\Q}{\sigma_\P}\right) + \frac{\sigma_\P^2 + (m_\Q-m_\P)^2}{2 \sigma_\Q^2} - \frac{1}{2}.
\end{align*}
\end{lemma}

\section{FURTHER EXPERIMENTAL DETAILS}\label{sec:CEMsupp}


In this section, we present the CEM algorithm and the transformation functions used to deal with constraints on the parameters.

\subsection{CEM Algorithm}
The CEM algorithm maximizing an objective function $L$ is given in Alg.~\ref{alg:CE}.
\begin{algorithm*}
   \caption{Maximization of $L$ with CEM}
   \label{alg:CE}
    \begin{algorithmic}[1]
      \STATE {\bfseries Input:} Initial value $\bm{\theta}^{(0)}$, quantile parameter $\rho>0$, smoothing parameter $\omega \in (0,1]$, sample size $S\in \N^*$, accuracy $\epsilon>0$, iterations number $T\in \N^*$.

        Initialize the iteration and the elite level: $t=1$, $\gamma_0=+\infty$.
        \REPEAT
	     \STATE Generate samples: $\{\mathbf{x}_s\}_{s\in[S]} \iid f\left(\cdot\,;\bm \theta^{(t-1)}\right)$.
	     \STATE Evaluate performance: $L_s = L(\mathbf{x}_s)$, $s\in [S]$.
	     \STATE Set level: $\gamma_t = L_{\left(\left\lceil (1-\rho)S \right\rceil  \right)}$ \COMMENT{$(1-\rho)$-quantile of $\{L(\mathbf{x}_s)\}_{s\in [S]}$}.
	     \STATE Estimate new parameter: $\tilde{\bm{\theta}}^{(t)} = \argmax_{\bm{\theta}  \in \bm{\Theta}} \frac{1}{S} \sum_{s\in [S]} \I_{\{L(\mathbf{x}_s)\ge \gamma_t\}} \log \left[f\left(\mathbf{x}_s; \bm \theta \right)\right]$ \COMMENT{MLE on the elite}.
	     \STATE Smoothing: $\bm{\theta}^{(t)} = (1-\omega) \bm{\theta}^{(t-1)} + \omega \tilde{\bm{\theta}}^{(t)}$.
      \UNTIL{$(t\le T)$ \AND ($\max(\lvert\gamma_t-\gamma_{t-1}\rvert, \lVert\bm  \theta^{(t)} - \bm \theta^{(t-1)}\rVert_\infty)\ge\epsilon$)}
      \STATE {\bfseries Output:} $\hat{\mathbf{x}} = \E_{\mathbf{x}\sim f\left(\cdot\,; \bm \theta^{(T)}\right)} \mathbf{x}$. 
    \end{algorithmic}
\end{algorithm*}

\subsection{Parameter Settings in our Experiments}

In this section, we detail how we used the CEM algorithm in our numerical experiments.
 
 \paragraph{CEM choice of hyperparameters:} In all our experiments,  we chose the pdf $ f\left(\cdot\,;\bm \theta\right)$ to be Gaussian with dimension adapted to the size of the problem,  with a mean value initialized taking into account the constraints on the parameters,  and a covariance matrix set to the identity: $\theta^{(t)} = \left( \bm \mu_t,  \bm \Sigma_t\right)$,  $\bm \mu_0\in \R^d$, $\bm \Sigma_0 = \b I_d \in \R^{d\times d}$. 
 We considered $N=150$, $\rho=0.1$ and set the maximum number of iterations to $T=30$.  A stopping criteria on the update of the elite parameter $\gamma_t$ and on the samples distribution parameter $\bm \theta^{(t)}$ was considered,  under the form $\max(\lvert\gamma_t-\gamma_{t-1}\rvert, \lVert\bm  \theta^{(t)} - \bm \theta^{(t-1)}\rVert_\infty)<\epsilon$, so that the algorithm stops when these parameters do not change much. The threshold in the stopping criteria  was conservatively set to $\epsilon=10^{-8}$. 
 
  \paragraph{CEM adaptation to constraints:} 
 Since the considered Gaussian distribution has unbounded support, and our experiments involve loss functions on specific domains of parameters,  a  transformation was applied on the samples to ensure this constraint.  Denoting by $g$ the transformation function,  the loss function  was evaluated on $\{g( \bm x_s)\}_{s\in[S]}$; we chose $g$ in CEM as follows. 
 
\begin{itemize}
             \item \tb{Experiment 2}: The index replication problem involved optimising the loss function on $ \W^3 :=  \left\{\bw \in \left(\Rnn\right)^3\,:\, \sum_{j=1}^3 w_j = 1\right\}. $ For $\bw\in \W^3$, by $w_3 = 1-w_1-w_2$, the problem can be reduced to $d=2$. This corresponds to taking $g((x_1,x_2)) = (x_1, x_2, 1-x_1-x_2)$.  We chose $\bm \mu_0=\left( \frac{1}{3},\frac{1}{3}\right)$.  The positivity constraint could also be enforced by considering the  transformation 
             \begin{equation*}
             g(\bm x) = \left( \frac{e^{x_1}}{1+\sum_{i=1}^2 e^{x_i}} ,  \frac{e^{x_2}}{1+\sum_{i=1}^2 e^{x_i}},  \frac{1}{1+\sum_{i=1}^2 e^{x_i}} \right),
             \end{equation*} 
             with $\bm \mu_0=(0,0)$ (to ensure that $g(\bm \mu_0) = \left( \frac{1}{3},\frac{1}{3}\right)$). We explored both options in our experiments. We got slightly better results  by just enforcing the sum-to-one constraint.
    \item \tb{Experiment 3}: Here, we performed the calibration of a beta distribution, for which the parameters have to be positive (or even smaller than one), and of a Gaussian and skew Gaussian distribution, for which the variance has to be positive.
    \begin{itemize}
    \item beta calibration: In our specific application of beta distribution on LGD ratios, the distribution was displaying a U-shape (see Fig.~\ref{fig:beta}), which is typical of a beta distribution with parameters $\alpha, \beta \in (0,1)$. To enforce this  constraint on $\alpha$ and  $\beta$,  the following sigmoid transformation from $\R^2$ to $(0,1)^2$ was applied:
      \begin{equation*}
             g(\bm x) = \left( \frac{e^{x_1}}{1+e^{x_1}} ,  \frac{e^{x_2}}{1+e^{x_2}} \right).
             \end{equation*}
       \item Gaussian and skew Gaussian calibration: To deal with positive variance (parameters $\sigma$ and $v$), the   transformation $x\mapsto e^{x}$ 
       mapping $\R$ to $\Rpp$ was applied on respective coordinates. 
    \end{itemize} 
\end{itemize}

\begin{figure}[h]
\vspace{.3in}
\centerline{\includegraphics[scale=0.5]{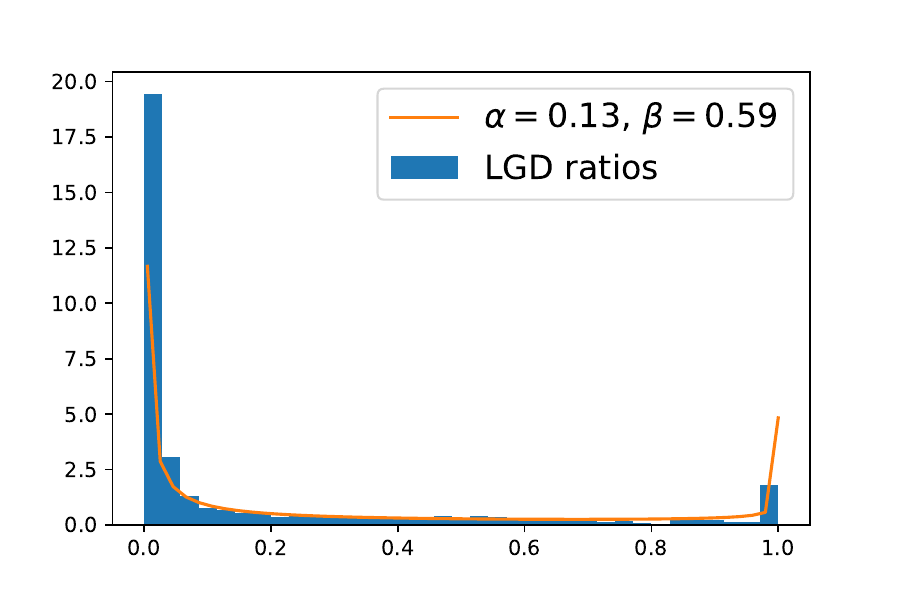}}
\vspace{.3in}
\caption{Histogram of LGD ratios and the calibrated beta distribution.}\label{fig:beta}
\end{figure}


\bibliography{BIB/collected,BIB/collected_plus}

\begin{thebibliography}{}

\bibitem[Albert et~al., 2022]{albert22adaptive}
Albert, M., Laurent, B., Marrel, A., and Meynaoui, A. (2022).
\newblock Adaptive test of independence based on {HSIC} measures.
\newblock {\em The Annals of Statistics}, 50(2):858--879.

\bibitem[Alquier and Gerber, 2024]{alquier2024universal}
Alquier, P. and Gerber, M. (2024).
\newblock Universal robust regression via maximum mean discrepancy.
\newblock {\em Biometrika}, 111(1):71--92.

\bibitem[Aronszajn, 1950]{aronszajn50theory}
Aronszajn, N. (1950).
\newblock Theory of reproducing kernels.
\newblock {\em Transactions of the American Mathematical Society}, 68:337--404.

\bibitem[Balasubramanian et~al., 2021]{balasubramanian21optimality}
Balasubramanian, K., Li, T., and Yuan, M. (2021).
\newblock On the optimality of kernel-embedding based goodness-of-fit tests.
\newblock {\em Journal of Machine Learning Research}, 22(1):1--45.

\bibitem[Bamberg and Wagner, 2000]{Bamberg2000replication}
Bamberg, G. and Wagner, N. (2000).
\newblock Equity index replication with standard and robust regression
  estimators.
\newblock {\em OR-Spektrum}, 22:525--543.

\bibitem[Baringhaus and Franz, 2004]{baringhaus04new}
Baringhaus, L. and Franz, C. (2004).
\newblock On a new multivariate two-sample test.
\newblock {\em Journal of Multivariate Analysis}, 88:190--206.

\bibitem[Baum et~al., 2023]{baum22kernel}
Baum, J., Kanagawa, H., and Gretton, A. (2023).
\newblock A kernel {S}tein test of goodness of fit for sequential models.
\newblock In {\em International Conference on Machine Learning (ICML)}, pages
  1936--1953.

\bibitem[Berlinet and Thomas-Agnan, 2004]{berlinet04reproducing}
Berlinet, A. and Thomas-Agnan, C. (2004).
\newblock {\em Reproducing Kernel Hilbert Spaces in Probability and
  Statistics}.
\newblock Kluwer.

\bibitem[Bishwal, 2007]{bishwal2007parameter}
Bishwal, J.~P. (2007).
\newblock {\em Parameter estimation in stochastic differential equations}.
\newblock Springer.

\bibitem[Black and Scholes, 1973]{blackscholes1973}
Black, F. and Scholes, M. (1973).
\newblock The pricing of options and corporate liabilities.
\newblock {\em Journal of Political Economy}, 81(3):637--654.

\bibitem[Borgwardt et~al., 2020]{borgwardt20graph}
Borgwardt, K., Ghisu, E., Llinares-L{\'o}pez, F., O'Bray, L., and Riec, B.
  (2020).
\newblock Graph kernels: State-of-the-art and future challenges.
\newblock {\em Foundations and Trends in Machine Learning}, 13(5-6):531--712.

\bibitem[Boucheron et~al., 2013]{boucheron2013concentration}
Boucheron, S., Lugosi, G., and Massart, P. (2013).
\newblock {\em Concentration inequalities}.
\newblock Oxford University Press, Oxford.

\bibitem[Briol et~al., 2019a]{briol19statistical}
Briol, F.-X., Barp, A., Duncan, A.~B., and Girolami, M. (2019a).
\newblock Statistical inference for generative models with maximum mean
  discrepancy.
\newblock Technical report.
\newblock (\url{https://arxiv.org/abs/1906.05944}).

\bibitem[Briol et~al., 2019b]{briol2019probabilistic}
Briol, F.-X., Oates, C.~J., Girolami, M., Osborne, M.~A., and Sejdinovic, D.
  (2019b).
\newblock {P}robabilistic integration: a role in statistical computation?
\newblock {\em Statistical Science. A Review Journal of the Institute of
  Mathematical Statistics}, 34(1):38--42.

\bibitem[Cam, 1973]{lecam73convergence}
Cam, L.~L. (1973).
\newblock Convergence of estimates under dimensionality restrictions.
\newblock {\em The Annals of Statistics}, 1:38--53.

\bibitem[Casella and Berger, 2024]{casella2024statistical}
Casella, G. and Berger, R. (2024).
\newblock {\em Statistical inference}.
\newblock CRC Press.

\bibitem[Chalabi and Wuertz, 2012]{chalabi12portfolio}
Chalabi, Y. and Wuertz, D. (2012).
\newblock Portfolio optimization based on divergence measures.
\newblock Technical report.
\newblock (\url{https://mpra.ub.uni-muenchen.de/43332/}).

\bibitem[Chen and Wang, 2013]{chen2013curve}
Chen, R. and Wang, Z. (2013).
\newblock Curve fitting of the corporate recovery rates: the comparison of beta
  distribution estimation and kernel density estimation.
\newblock Technical report.
\newblock (\url{https://www.ncbi.nlm.nih.gov/pmc/articles/PMC3707900/}).

\bibitem[Christopher~Adcock and Loperfido, 2015]{Adcock2015skewfinance}
Christopher~Adcock, M.~E. and Loperfido, N. (2015).
\newblock Skewed distributions in finance and actuarial science: a review.
\newblock {\em The European Journal of Finance}, 21(13-14):1253--1281.

\bibitem[Cont, 2001]{cont01empirical}
Cont, R. (2001).
\newblock Empirical properties of asset returns: Stylized facts and statistical
  issues.
\newblock {\em Quantitative Finance}, 1(2):223--236.

\bibitem[Deb et~al., 2020]{deb2020measuring}
Deb, N., Ghosal, P., and Sen, B. (2020).
\newblock Measuring association on topological spaces using kernels and
  geometric graphs.
\newblock Technical report.
\newblock (\url{https://arxiv.org/pdf/2010.01768/}).

\bibitem[Duchi, 2007]{duchi07derivations}
Duchi, J. (2007).
\newblock Derivations for linear algebra and optimization.
\newblock Technical report, The University of Stanford.
\newblock (\url{https://web.stanford.edu/~jduchi/projects/general_notes.pdf}).

\bibitem[Fortin and Kuzmics, 2002]{fortin2002tail}
Fortin, I. and Kuzmics, C. (2002).
\newblock Tail-dependence in stock-return pairs.
\newblock {\em Intelligent Systems in Accounting, Finance \& Management},
  11(2):89--107.

\bibitem[Fukumizu et~al., 2008]{fukumizu08kernel}
Fukumizu, K., Gretton, A., Sun, X., and Sch{\"o}lkopf, B. (2008).
\newblock Kernel measures of conditional dependence.
\newblock In {\em Advances in Neural Information Processing Systems (NIPS)},
  pages 498--496.

\bibitem[Gretton et~al., 2012]{gretton12kernel}
Gretton, A., Borgwardt, K., Rasch, M., Sch{\"o}lkopf, B., and Smola, A. (2012).
\newblock A kernel two-sample test.
\newblock {\em Journal of Machine Learning Research}, 13(25):723--773.

\bibitem[Gretton et~al., 2008]{gretton08kernel}
Gretton, A., Fukumizu, K., Teo, C.~H., Song, L., Sch{\"o}lkopf, B., and Smola,
  A. (2008).
\newblock A kernel statistical test of independence.
\newblock In {\em Advances in Neural Information Processing Systems (NIPS)},
  pages 585--592.

\bibitem[Gupton and Stein, 2002]{gupton2002losscalctm}
Gupton, G. and Stein, R. (2002).
\newblock Losscalctm: Model for predicting loss given default (lgd).

\bibitem[Hagrass et~al., 2024]{hagrass2024spectral}
Hagrass, O., Sriperumbudur, B., and Li, B. (2024).
\newblock Spectral regularized kernel two-sample tests.
\newblock {\em The Annals of Statistics}, 52(3):1076--1101.

\bibitem[Hall and Heyde, 1980]{hall2014martingale}
Hall, P. and Heyde, C.~C. (1980).
\newblock {\em Martingale limit theory and its application}.
\newblock Academic Press, Inc. [Harcourt Brace Jovanovich, Publishers], New
  York-London.

\bibitem[Haussler, 1999]{haussler99convolution}
Haussler, D. (1999).
\newblock Convolution kernels on discrete structures.
\newblock Technical report, University of California at Santa Cruz.
\newblock
  (\url{http://cbse.soe.ucsc.edu/sites/default/files/convolutions.pdf}).

\bibitem[Hoeffding, 1963]{hoeffding1963}
Hoeffding, W. (1963).
\newblock Probability inequalities for sums of bounded random variables.
\newblock {\em Journal of the American Statistical Association},
  58(301):13--30.

\bibitem[Jadhav and Ramanathan, 2009]{jadhav2009parametric}
Jadhav, D. and Ramanathan, T. (2009).
\newblock Parametric and non-parametric estimation of value-at-risk.
\newblock {\em The Journal of Risk Model Validation}, 3(1):51--71.

\bibitem[Kennedy, 1998]{kennedy1998bayesian}
Kennedy, M. (1998).
\newblock Bayesian quadrature with non-normal approximating functions.
\newblock {\em Statistics and Computing}, 8(4):365--375.

\bibitem[Kir{\'a}ly and Oberhauser, 2019]{kiraly19kernel}
Kir{\'a}ly, F.~J. and Oberhauser, H. (2019).
\newblock Kernels for sequentially ordered data.
\newblock {\em Journal of Machine Learning Research}, 20:1--45.

\bibitem[Klebanov, 2005]{klebanov05ndistance}
Klebanov, L. (2005).
\newblock {\em N-Distances and Their Applications}.
\newblock Charles University, Prague.

\bibitem[Kroese, 2004]{rubinstein04cross}
Kroese, R. Y. R. D.~P. (2004).
\newblock {\em The Cross-Entropy Method}.
\newblock Springer.

\bibitem[Lassance, 2019]{lassance19information}
Lassance, N. (2019).
\newblock {\em Information-theoretic approaches to portfolio selection}.
\newblock PhD thesis, Louvain School of Management.

\bibitem[M{\"u}ller, 1997]{muller97integral}
M{\"u}ller, A. (1997).
\newblock Integral probability metrics and their generating classes of
  functions.
\newblock {\em Advances in Applied Probability}, 29:429--443.

\bibitem[Paulsen and Raghupathi, 2016]{paulsen16introduction}
Paulsen, V.~I. and Raghupathi, M. (2016).
\newblock {\em An Introduction to the Theory of Reproducing Kernel Hilbert
  Spaces}.
\newblock Cambridge University Press.

\bibitem[Peyr\'e and Cuturi, 2019]{peyre19computational}
Peyr\'e, G. and Cuturi, M. (2019).
\newblock Computational optimal transport.
\newblock {\em Foundations and Trends in Machine Learning}, 11(5-6):355--607.

\bibitem[Pitcan, 2017]{pitcan2017note}
Pitcan, Y. (2017).
\newblock {A Note on Concentration Inequalities for U-Statistics}.
\newblock Technical report, University of Berkeley.
\newblock (\url{https://arxiv.org/abs/1712.06160}).

\bibitem[Roncalli and Weisang, 2009]{roncalli2009tracking}
Roncalli, T. and Weisang, G. (2009).
\newblock Tracking problems, hedge fund replication and alternative beta.
\newblock Technical report.
\newblock (\url{https://ssrn.com/abstract=1325190}).

\bibitem[Schrab et~al., 2022]{schrab22efficient}
Schrab, A., Kim, I., Guedj, B., and Gretton, A. (2022).
\newblock Efficient aggregated kernel tests using incomplete {U}-statistics.
\newblock In {\em Advances in Neural Information Processing Systems (NeurIPS)},
  pages 18793--18807.

\bibitem[Sejdinovic et~al., 2013]{sejdinovic13equivalence}
Sejdinovic, D., Sriperumbudur, B., Gretton, A., and Fukumizu, K. (2013).
\newblock Equivalence of distance-based and {RKHS}-based statistics in
  hypothesis testing.
\newblock {\em Annals of Statistics}, 41:2263--2291.

\bibitem[Smola et~al., 2007]{smola07hilbert}
Smola, A., Gretton, A., Song, L., and Sch{\"o}lkopf, B. (2007).
\newblock A {H}ilbert space embedding for distributions.
\newblock In {\em Algorithmic Learning Theory (ALT)}, pages 13--31.

\bibitem[Song et~al., 2008]{song08tailoring}
Song, L., Zhang, X., Smola, A., Gretton, A., and Sch\"{o}lkopf, B. (2008).
\newblock Tailoring density estimation via reproducing kernel moment matching.
\newblock In {\em International Conference on Machine Learning (ICML)}, page
  992–999.

\bibitem[Sriperumbudur et~al., 2010]{sriperumbudur10hilbert}
Sriperumbudur, B., Gretton, A., Fukumizu, K., Sch{\"o}lkopf, B., and Lanckriet,
  G. (2010).
\newblock Hilbert space embeddings and metrics on probability measures.
\newblock {\em Journal of Machine Learning Research}, 11:1517--1561.

\bibitem[Steinwart and Christmann, 2008]{steinwart08support}
Steinwart, I. and Christmann, A. (2008).
\newblock {\em Support Vector Machines}.
\newblock Springer.

\bibitem[Szab{\'o} and Sriperumbudur, 2018]{szabo2018characteristic}
Szab{\'o}, Z. and Sriperumbudur, B. (2018).
\newblock Characteristic and universal tensor product kernels.
\newblock {\em Journal of Machine Learning Research}, 18(233):1--29.

\bibitem[Sz{\'e}kely and Rizzo, 2004]{szekely04testing}
Sz{\'e}kely, G. and Rizzo, M. (2004).
\newblock Testing for equal distributions in high dimension.
\newblock {\em InterStat}, 5:1249--1272.

\bibitem[Sz{\'e}kely and Rizzo, 2005]{szekely05new}
Sz{\'e}kely, G. and Rizzo, M. (2005).
\newblock A new test for multivariate normality.
\newblock {\em Journal of Multivariate Analysis}, 93:58--80.

\bibitem[Tolstikhin et~al., 2016]{tolstikhin16minimax}
Tolstikhin, I., Sriperumbudur, B., and Sch{\"o}lkopf, B. (2016).
\newblock Minimax estimation of maximal mean discrepancy with radial kernels.
\newblock In {\em Advances in Neural Information Processing Systems (NIPS)},
  pages 1930--1938.

\bibitem[Tsybakov, 2009]{tsybakov08introduction}
Tsybakov, A.~B. (2009).
\newblock {\em Introduction to nonparametric estimation}.
\newblock Springer.

\bibitem[Zinger et~al., 1992]{zinger92characterization}
Zinger, A., Kakosyan, A., and Klebanov, L. (1992).
\newblock A characterization of distributions by mean values of statistics and
  certain probabilistic metrics.
\newblock {\em Journal of Soviet Mathematics}.

\bibitem[Zolotarev, 1983]{zolotarev83probability}
Zolotarev, V. (1983).
\newblock Probability metrics.
\newblock {\em Theory of Probability and its Applications}, 28:278--302.

\end{thebibliography}

\end{document}